\documentclass[11pt]{article}
\usepackage[letterpaper,margin=1in]{geometry}
\usepackage{times}
\usepackage[T1]{fontenc}
\usepackage[round]{natbib}
\usepackage{authblk}
\usepackage{amsmath,amssymb,amsthm,graphicx,booktabs}
\usepackage{xcolor}
\usepackage{colortbl}
\definecolor{toySourceOne}{HTML}{376BA2}
\definecolor{toySourceTwo}{HTML}{008477}
\definecolor{toySourceThree}{HTML}{8064A2}
\newcommand{\toyone}[1]{\cellcolor{toySourceOne!16}#1}
\newcommand{\toytwo}[1]{\cellcolor{toySourceTwo!16}#1}
\newcommand{\toythree}[1]{\cellcolor{toySourceThree!16}#1}

\usepackage{algorithm,algpseudocode}
\usepackage{etoc}
\usepackage{hyperref,url}
\hypersetup{hidelinks,pdftitle={Identifying Neural Source Dynamics from Unknown Local Interventions},pdfauthor={Ayana Mussabayeva, Jiaqi Sun, Anuar Aimoldin, Olivier Oullier, Kun Zhang},pdfsubject={Anatomically calibrated identification of linear source dynamics}}
\newtheorem{theorem}{Theorem}
\newtheorem{proposition}[theorem]{Proposition}
\newtheorem{lemma}[theorem]{Lemma}
\newtheorem{corollary}[theorem]{Corollary}
\theoremstyle{definition}

\DeclareMathOperator{\rank}{rank}
\DeclareMathOperator{\diag}{diag}

\newcommand{\R}{\mathbb R}
\newcommand{\Frob}[1]{\lVert #1\rVert_{\mathrm F}}
\title{Identifying Neural Source Dynamics\\from Unknown Local Interventions}
\author[1]{Ayana Mussabayeva}
\author[2]{Jiaqi Sun}
\author[1]{Anuar Aimoldin}
\author[1]{Olivier Oullier}
\author[1,2]{Kun Zhang}
\affil[1]{Mohamed bin Zayed University of Artificial Intelligence (MBZUAI), Abu Dhabi, UAE}
\affil[2]{Carnegie Mellon University (CMU), Pittsburgh, PA, USA}
\affil[ ]{\small\texttt{\{ayana.mussabayeva,anuar.aimoldin,olivier.oullier,kun.zhang\}@mbzuai.ac.ae}\\\texttt{jiaqisun@andrew.cmu.edu}}
\date{}
\begin{document}
\etocdepthtag.toc{main}
\maketitle

\begin{abstract}
Electroencephalography (EEG) records mixtures of brain-source activity. Even with a known anatomical forward model, experiments that excite only part of the source-state space leave the dynamics unidentified, and repetition cannot resolve the ambiguity. We show that unknown local mechanism changes can supply the missing information. We consider linear dynamics among fixed anatomical sources with known source-state initialization patterns. Changing one source's update rule for one transition leaves a rank-one, source-specific signature in subsequent EEG: subtracting matched baseline responses isolates it, and the forward model identifies the source and calibrates its response history. Combining these histories with initialization responses recovers source interactions without baseline reachability and without first identifying the intervention coefficients. We establish sufficient recovery conditions, a direct estimator, and a noise-sensitivity bound conditional on correct source labels. Simulated EEG on anatomy derived from magnetic resonance imaging confirms the information gain: with baseline excitation confined to four of twelve source coordinates, eight unknown changes recover all dynamics in 32/32 systems, whereas baseline realization, baseline regression through an invertible forward model, and changes that leave the tested states unexposed all fail, and explicitly constructed alternative dynamics reproduce every baseline mean. Where baseline information suffices, direct reconstruction is also more reliable than a matched-information spectral estimator. Nonlocal changes and forward-model error limit accuracy even when source labels are correct.
\end{abstract}

\section{Introduction}
\label{sec:v6introduction}

Learning neural interactions requires more than knowing how sources appear at the sensors. The known anatomical forward model, or \emph{leadfield}, maps source activity to electroencephalography (EEG). Yet distinct dynamics can produce identical mean responses if initialized states and their baseline evolution explore only part of the source space. More repetitions improve precision but cannot resolve this ambiguity, and neither can a better readout: even exact recovery of mean source states leaves the dynamics on unexcited directions undetermined. Classical realization theory makes this limit explicit, since full-order recovery from input--output data requires reachability as well as observability \citep{ho1966}.

Perturbational neuroscience offers a way past it. A local intervention changes how one circuit element responds and can push activity into directions that baseline experiments never reach \citep{wagenmaker2024active,premoli2014paired}. In EEG, however, its effect arrives mixed across sensors, and the intervention itself is rarely calibrated: which anatomical source it altered, and by how much, are unknown. We ask: \emph{can local mechanism changes reveal dynamics absent from baseline mean responses and identify them in anatomical coordinates, even when their targets and coefficients are unknown?}

For example, initializing source 1 leaves source 2's mean activity at zero if baseline dynamics never transfer activity to it; no repetition budget or sensor coverage then reveals how source 2 acts on the rest. Strengthening that connection for a single transition exposes source 2's influence. We study linear source dynamics among fixed anatomical sources with known initialization patterns and an unknown change to one source's receiving rule for one transition. Unlike an additive pulse, this change acts on the current state; the original dynamics resume afterwards.

The key insight is that \emph{a source-response history can be identifiable even when the intervention itself is not}. Subtracting matched baseline from intervened means isolates a displacement at one source and its subsequent propagation. For each intervention mode, conditions change the displacement's magnitude but not the response shape. This rank-one structure reveals the history up to scale; its immediate sensor pattern identifies the source and fixes that scale through the leadfield. Combining these calibrated histories with initialization responses then recovers the dynamics under source-coverage and temporal-observability conditions, without full baseline reachability and without first recovering the intervention coefficients.

Our contribution is a constructive identification result and an estimator that realizes it. Sufficient conditions separate intervention exposure, source coverage, anatomical distinguishability and temporal observability; none requires full baseline reachability (Section~\ref{sec:v6identification}). A direct estimator extracts and combines calibrated histories, with a dynamics-error bound that separates coverage conditioning from weak temporal observability (Section~\ref{sec:v6estimation}). Simulated EEG on four MRI-derived Localize-MI anatomies \citep{mikulan2020} confirms the information gain: with baseline excitation confined to four of twelve source coordinates at every horizon, eight unknown local changes recover all twelve-source dynamics in 32/32 systems, whereas baseline realization, baseline regression through an invertible leadfield, and changes that leave the tested states unexposed all fail in 0/32, and explicitly constructed alternative dynamics reproduce every baseline mean (Section~\ref{sec:v6partial}). Where baseline information suffices, direct reconstruction remains more reliable than a matched-information spectral factorization (155/160 versus 141/160) and initializes likelihood fitting (Section~\ref{sec:v6spectral}). We close by delimiting what the guarantees do not cover: noise, nonlocal changes and forward-model error, under which correct source labels need not imply correct dynamics, and what recorded EEG can currently test (Section~\ref{sec:v6controls}). Our target is dynamics among predefined anatomical sources, not unrestricted spatial localization; the construction applies to any known linear observation map under the same conditions, although only EEG is examined here.

\section{Related work}
\label{sec:v6related}

Identifying hidden-source dynamics also requires fixing their coordinates. Non-Gaussian state-space models provide observational identification \citep{zhang2011}, while interventional causal representation learning includes unknown-target settings under assumptions on mixing and mechanism changes \citep{squires2023,zhang2024}; temporal context can resolve ambiguities under noninvertible observations \citep{chen2024}. We instead use a known forward model to fix anatomical source labels and units, so no statistical assumption on the mixing is needed.

Fixed coordinates, however, do not ensure sufficient excitation, and neither classical nor interventional identification covers our regime. Input--output realization requires reachability and observability for full-order recovery \citep{ho1966}; interventional identification through varying input distributions assumes controllability \citep{rajendran2024}. Our local receiving-mechanism contrasts remove the reachability requirement: the leadfield labels and scales each exposed direction, so unknown changes can complete a coverage condition that baseline inputs cannot. The matched-information spectral-local comparator instead factors baseline responses into a latent realization before imposing anatomical constraints, and so inherits the reachability requirement.

Known-leadfield state-space methods estimate source interactions by integrating observation physics with latent dynamics \citep{cheung2010statespace,soleimani2022nlgc}; our contrasts can initialize such iterative likelihood fitting (Appendix~\ref{app:v6em}). Dynamic causal modelling distinguishes driving inputs from coupling modulation \citep{friston2003}; our model restricts modulation to an unknown change in one source's receiving rule for one transition, followed by the same baseline dynamics. Appendix~\ref{app:v6related} provides further comparisons.

\section{Problem formulation}
\label{sec:v6setup}

Let $z_\tau\in\R^q$ contain amplitudes at $q$ fixed source locations and orientations, and $x_\tau\in\R^m$ contain $m$ independent EEG contrasts at episode time $\tau$.

We study
\begin{equation}
 z_{\tau+1}=(F+D_{e_\tau})z_\tau+\epsilon_\tau,
 \qquad x_\tau=Lz_\tau+\nu_\tau.
 \label{eq:v6model}
\end{equation}
The unknown $F\in\R^{q\times q}$ is the baseline transition matrix. The known leadfield $L\in\R^{m\times q}$ fixes source ordering and units; its column $L_{:,j}$ is the sensor pattern of a unit source $j$. A colon selects all entries on that axis. Our principal setting has $\rank L<q$, so one sensor observation need not determine the state.

$F$ describes effective interactions within a fixed operating regime, not globally linear brain dynamics. Linear models are used for EEG and magnetoencephalography (MEG) source dynamics \citep{zhang2011,cheung2010statespace} and optogenetically evoked population responses \citep{wagenmaker2024active}. Locally, $z_\tau$ may represent deviations from a common equilibrium, with $F$ the Jacobian of a smooth baseline map, provided trajectories remain where the Jacobian varies little. Our guarantees assume the stated linear model and intervention protocol; our simulations do not test nonlinear source dynamics.

\begin{figure}[t]
 \centering
 \includegraphics[width=\linewidth]{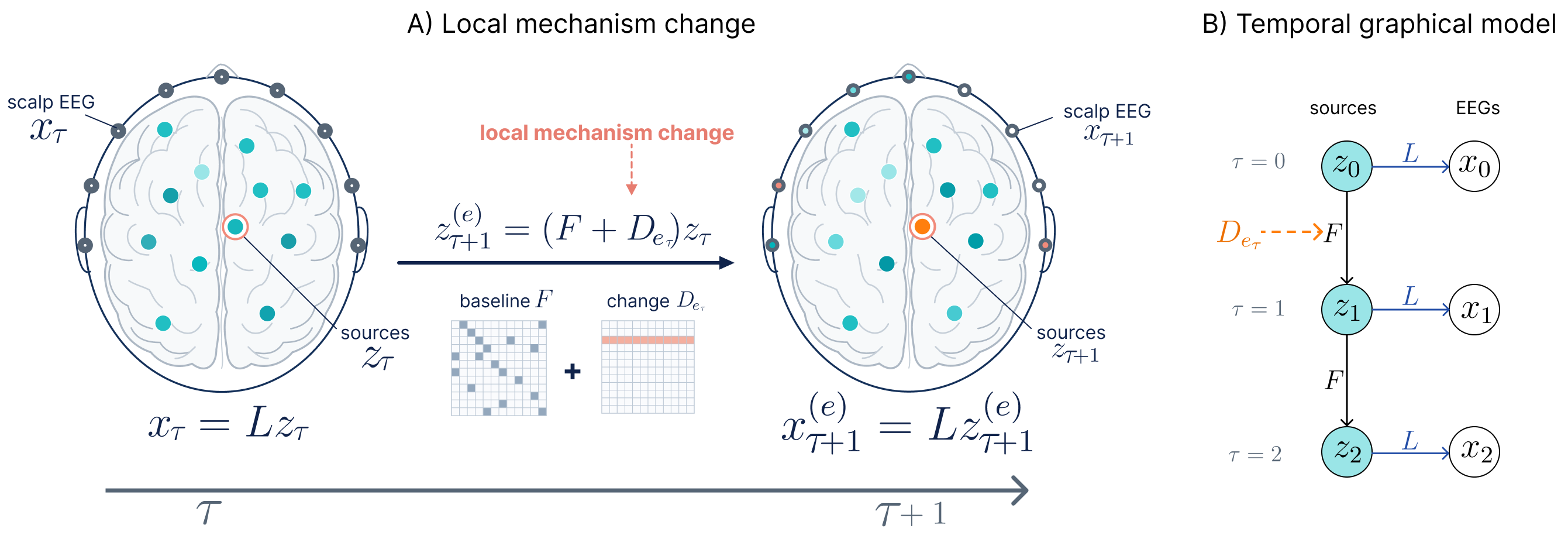}
 \caption{Source dynamics and EEG. (A) One receiving row of $F$ changes; $L$ stays fixed. (B) The first transition uses $F+D_{e_0}$; $F$ then resumes without resetting the state. Nodes in (B) are vectors of source amplitudes $z_\tau$ or independent EEG contrasts $x_\tau$. Intervention superscripts are omitted from labels; colors and entries are schematic; noise is omitted.}
 \label{fig:v6state-change}
\end{figure}

\paragraph{Initialization and intervention.}
Episodes start with $\mathbb E[z_0\mid u]=Ku$, where known $K\in\R^{q\times r}$ specifies source-state patterns and $u\in\R^r$ their controlled amplitudes. There is no subsequent baseline additive drive. The exogenous schedule selects baseline $D_0=0$ or one of $E$ local modes:
\begin{equation}
 D_e=e_{j_e}v_e^\top,\qquad v_e\in\R^q\setminus\{0\}.
 \label{eq:v6row}
\end{equation}
Here $e_j$ is the $j$th coordinate vector; both target $j_e$ and coefficients $v_e$ are unknown. The state is unchanged at onset, one transition uses $F+D_e$, and all other transitions use the same $F$. The resulting state is not reset; $L$ stays fixed. Innovations and sensor noise have zero conditional mean given initialization and schedule. Known $K$ means calibrated source states, not merely a known stimulation contact.

\paragraph{Observed response matrices.}
Choose an output horizon $T\ge2$ and insertion times $\tau_{\mathrm{int}}=0,\ldots,s-1$, with $s\ge1$. For each unit input $u^{(i)}\in\R^r$, stack the mean EEG at times $\tau_{\mathrm{int}}+1+t$, $t=0,\ldots,T-1$, after inserting mode $e$. This is one column of $H^{[e]}\in\R^{mT\times rs}$: rows index sensors and output lags, columns index initialization and insertion time. Trial averages estimate $H^{[e]}$ without observing $z$ or knowing $F$. The matched baseline $H_+$ uses the same inputs and observation times without intervention; $H_0$ uses baseline times one sample earlier (Appendix~\ref{app:v6setup}).

To express these measured histories through the model, define
\begin{equation}
 O_T=\begin{bmatrix}L\\LF\\\vdots\\LF^{T-1}\end{bmatrix}\in\R^{mT\times q},
 \qquad R_s=[K,FK,\ldots,F^{s-1}K]\in\R^{q\times rs}.
 \label{eq:v6histories}
\end{equation}
$O_T$ contains unit-source histories; $R_s$ contains mean states reached by baseline initializations. Then
\begin{equation}
 H_0=O_TR_s,\qquad H_+=O_TFR_s,\qquad
 H^{[e]}=O_T(F+D_e)R_s.
 \label{eq:v6blocks}
\end{equation}
For example, initializing with $u^{(i)}$ and changing the first transition gives the first post-transition sample $x_1$ with mean $L(F+D_e)K_{:,i}$. Its matched baseline mean is $LFK_{:,i}$; their difference is $LD_eK_{:,i}$. Later samples track the propagation of that difference by $F$. Thus the $H$ matrices store measured conditional means, whereas $O_T$ describes unit-source responses that must be recovered. Shared baseline lags reuse the same observations and are not independent entries.

\section{Identification from local interventions}
\label{sec:v6identification}

\subsection{Identifying source-response histories}

What can an unknown local change reveal? It creates a displacement in one source coordinate relative to baseline. Its amplitude depends on the tested state, but its subsequent propagation follows the same $F$. Consequently,
\begin{equation}
 \Delta H_e:=H^{[e]}-H_+
 =O_TD_eR_s
 =\underbrace{O_Te_{j_e}}_{o_{j_e}}\,
  \underbrace{v_e^\top R_s}_{\psi_e^\top}.
 \label{eq:v6contrast}
\end{equation}
Here $o_j\in\R^{mT}$ is the unit-source history and $\psi_e\in\R^{rs}$ is the intervention's exposure across conditions. If $\psi_e\ne0$ and $L_{:,j_e}\ne0$, the contrast has rank one and identifies the line spanned by $o_{j_e}$. Its first sensor block is proportional to $L_{:,j_e}$. Nonzero, pairwise nonproportional leadfield columns identify the target and calibrate the entire history. The recoverable object is therefore the response history, not necessarily the intervention row: distinct coefficients with the same action on $R_s$ give the same contrast.

\subsection{From histories to dynamics}

When do the recovered histories determine $F$? Let $J$ be the distinct exposed targets and $E_J=[e_j]_{j\in J}$. Collect the available directions and their responses:
\begin{equation}
 \Omega=[K,E_J],\qquad Y=[O_TK,O_TE_J]=O_T\Omega.
 \label{eq:v6coverage}
\end{equation}
The first $r$ columns of $H_0$ supply $O_TK$; calibrated contrasts supply $O_TE_J$. Thus $Y$ is known without the intervention coefficients. The anchoring matrix $\Omega\in\R^{q\times(r+|J|)}$ measures source-coordinate coverage. Removing the last or first sensor block from $O_T$ gives
\[
 O_-=[L^\top,(LF)^\top,\ldots,(LF^{T-2})^\top]^\top,\qquad
 O_+=[(LF)^\top,\ldots,(LF^{T-1})^\top]^\top.
\]
Both have size $m(T-1)\times q$, and $O_+=O_-F$: each temporal block in $O_+$ is the corresponding block in $O_-$ multiplied by $F$. For $T=2$, these reduce to $L$ and $LF$.

\begin{theorem}[Constructive anatomical identification]
\label{thm:v6identification}
Under the model and schedule of Section~\ref{sec:v6setup}, suppose the columns of $L$ are nonzero and pairwise nonproportional. Selected changes with $v_e^\top R_s\ne0$ identify their targets and columns $O_Te_{j_e}$. If
\begin{equation}
 \rank\Omega=q,\qquad \rank O_-=q,
 \label{eq:v6conditions}
\end{equation}
the measured responses uniquely determine
\begin{equation}
 O_T=Y\Omega^\dagger,\qquad F=O_-^\dagger O_+.
 \label{eq:v6reconstruction}
\end{equation}
Here $\dagger$ is the Moore--Penrose pseudoinverse. Full rank of $R_s$ is not required.
\end{theorem}

\begin{proof}[Proof sketch]
Equation~\eqref{eq:v6contrast} supplies labelled, scaled columns. Full row rank of $\Omega$ gives $\Omega\Omega^\dagger=I_q$, recovering $O_T$. Full column rank of $O_-$ then solves the shift identity for $F$. Appendix~\ref{app:v6theory} gives full proofs.
\end{proof}

The conditions separate exposure ($v_e^\top R_s\ne0$), anatomical distinguishability through $L$, source coverage ($\rank\Omega=q$), and temporal observability ($\rank O_-=q$). Exposure produces a contrast; distinguishability identifies and calibrates its source. Coverage combines targets with possibly mixed initialization patterns, while observability lets temporal histories distinguish states that a single observation cannot. These conditions presuppose calibrated $L,K$, matched episodes and single-transition receiving-row changes; estimated ranks are diagnostics, not certificates.

These conditions do not require full baseline reachability. Baseline realization factors $H_0=O_TR_s$ and needs both factors to have rank $q$; our construction needs only each selected change to act on a tested state, $v_e^\top R_s\ne0$, and the revealed targets to complete $\Omega$. A change can open a path into an unreachable coordinate; suppressing an entirely unreachable row cannot. Appendix~\ref{app:v6gauge} characterizes the residual coordinate family when anchoring is incomplete.

Full column rank of $L$ makes instantaneous states distinguishable and ensures $\rank O_-=q$, but leaves dynamics on unexcited directions undetermined: an invertible readout recovers states, not the rows of $F$ that no tested state ever exercises. Section~\ref{sec:v6partial} tests this distinction empirically; Appendix~\ref{app:v6injective} treats square systems.

\paragraph{Illustration: recovery without baseline reachability.}
Consider $q=3$, $m=2$, $K=e_1$, $T=3$ and $s=2$. Choose dynamics with $F^\tau K=2^{-\tau}e_1$, so baseline means leave sources 2 and 3 unexcited. Two changes with unknown targets and coefficients reveal $o_2$ and $o_3$ after anatomical calibration. One contrast and the completed histories are:
\begin{equation}
 \renewcommand{\toyone}[1]{\cellcolor[HTML]{80B3E6}#1}
 \renewcommand{\toytwo}[1]{\cellcolor[HTML]{D6EEC8}#1}
 \renewcommand{\toythree}[1]{\cellcolor[HTML]{F6B6AA}#1}
\begingroup
\setlength{\arraycolsep}{3pt}
\renewcommand{\arraystretch}{1.08}
R_2=\left[\begin{array}{cc}
 \toyone{1}&\toyone{0.5}\\
 0&0\\
 0&0
\end{array}\right],\quad
\Delta H_1=\left[\begin{array}{cc}
 \toytwo{0}&\toytwo{0}\\
 \toytwo{2}&\toytwo{1}\\\hline
 \toytwo{2}&\toytwo{1}\\
 \toytwo{0}&\toytwo{0}\\\hline
 \toytwo{1}&\toytwo{0.5}\\
 \toytwo{0}&\toytwo{0}
\end{array}\right],\quad
L=\left[\begin{array}{ccc}
 \toyone{1}&\toytwo{0}&\toythree{1}\\
 \toyone{0}&\toytwo{1}&\toythree{1}
\end{array}\right],\quad
O_3=\left[\begin{array}{ccc}
 \toyone{1}&\toytwo{0}&\toythree{1}\\
 \toyone{0}&\toytwo{1}&\toythree{1}\\\hline
 \toyone{0.5}&\toytwo{1}&\toythree{0}\\
 \toyone{0}&\toytwo{0}&\toythree{1}\\\hline
 \toyone{0.25}&\toytwo{0.5}&\toythree{1}\\
 \toyone{0}&\toytwo{0}&\toythree{0}
\end{array}\right].
\endgroup

 \label{eq:v6worked-matrices}
\end{equation}
Blue, green and red identify sources 1, 2 and 3, respectively. Columns of $R_2,\Delta H_1$ index insertion times 0 and 1; columns of $L,O_3$ index sources. Horizontal rules separate the two-sensor output-lag blocks. The displayed contrast reveals $o_2$; the second mode supplies $o_3$.

Here $\rank R_s=1$ at every horizon and $\rank L=2$, yet $\rank O_-=3$. The known initialization and identified targets give $\Omega=I_3$, so Theorem~\ref{thm:v6identification} uniquely recovers $F$ while some intervention coefficients remain undetermined. The theorem permits mixed initialization patterns and any dynamics satisfying its rank conditions. This illustration uses exact means; Appendix~\ref{app:v6worked-example} gives the generating system and complete calculation. Section~\ref{sec:v6partial} repeats the same construction on twelve-source anatomical systems with noise.

\section{Estimation and stability}
\label{sec:v6estimation}

\subsection{Direct contrast reconstruction}
\label{sec:v6direct-estimation}

Equation~\eqref{eq:v6contrast} makes every nonzero population contrast rank one, with all columns proportional to one source history. Noise breaks this exact structure. We estimate the common direction using the leading unit left singular vector $\widehat u_e$ of the measured contrast, obtained by singular value decomposition (SVD). Its first sensor block $\widehat u_{e,\mathrm{top}}\in\R^m$ is used to estimate the source label and calibrate the history:
\begin{equation}
 \widehat j_e=\arg\max_j
 \frac{|\widehat u_{e,\mathrm{top}}^\top L_{:,j}|}
 {\|\widehat u_{e,\mathrm{top}}\|_2\|L_{:,j}\|_2},
 \quad
 \gamma_e=\frac{\widehat u_{e,\mathrm{top}}^\top L_{:,\widehat j_e}}
 {\|\widehat u_{e,\mathrm{top}}\|_2^2},
 \quad \widehat o^{(e)}_{\widehat j_e}=\gamma_e\widehat u_e.
 \label{eq:v6calibration}
\end{equation}
Absolute alignment handles the SVD sign ambiguity; the signed scale restores source units. Algorithm~\ref{alg:v6direct} combines the calibrated histories with initialization responses, completes $O_T$, and recovers $F$ through the temporal shift. It retains one paired column per intervention, including repeated targets. Here $\operatorname{SVD}_1$ returns the leading unit left singular vector and singular value; $\epsilon_\Delta,\epsilon_u$ guard against a degenerate contrast or first sensor block.

\begin{algorithm}[tbp]
\caption{Direct contrast reconstruction of anatomical source dynamics}
\label{alg:v6direct}
\begin{algorithmic}[1]
\Require Known $L,K$, measured $\widehat H_0,\widehat H_+,\{\widehat H^{[e]}\}_{e=1}^E$, thresholds $\epsilon_\Delta,\epsilon_u$.
\State $\Omega\gets K$,\quad $\widehat Y\gets$ first $r$ columns of $\widehat H_0$
\For{$e=1,\ldots,E$}
 \State $(\widehat u_e,\widehat\sigma_e)\gets\operatorname{SVD}_1(\widehat H^{[e]}-\widehat H_+)$
 \State \textbf{if} $\widehat\sigma_e\le\epsilon_\Delta$ \textbf{or} $\|\widehat u_{e,\rm top}\|_2\le\epsilon_u$: \Return \textsc{Invalid}
 \State Compute $\widehat j_e,\gamma_e$ using Eq.~\eqref{eq:v6calibration}
 \State $\Omega\gets[\Omega,e_{\widehat j_e}]$,\quad $\widehat Y\gets[\widehat Y,\gamma_e\widehat u_e]$
\EndFor
\State \textbf{if} $\rank\Omega<q$: \Return \textsc{Invalid}
\State $\widehat O_T\gets\widehat Y\Omega^\dagger$, set its first block to $L$ and form $\widehat O_-,\widehat O_+$
\State \textbf{if} $\rank\widehat O_-<q$: \Return \textsc{Invalid} \textbf{else}: \Return $\widehat F=\widehat O_-^\dagger\widehat O_+$
\end{algorithmic}
\end{algorithm}

Theorem~\ref{thm:v6identification} needs only a sufficient selected set of changes; the tested implementation uses all supplied modes, performs no mode selection, and returns an invalid fit when any mode fails its numerical guard. These guards are not statistical detection thresholds (Appendix~\ref{app:v6algorithm}); invalid fits count as failures.

\subsection{Conditional stability}
\label{sec:v6conditional-stability}

Correctly labelled histories can still yield inaccurate dynamics when noise is amplified during reconstruction: when the available directions cover some source combinations only weakly, or when distinct source states produce similar sensor histories. The following bound separates these two effects. Let $\|\cdot\|_2$ denote spectral norm for matrices and Euclidean norm for vectors, and $\sigma_{\min}$ the smallest singular value.
\begin{proposition}[Conditional reconstruction stability]
\label{prop:v6stability}
Under Eq.~\eqref{eq:v6conditions}, conditional on correct target labels, define
\begin{equation}
 \delta=\frac{\|\widehat Y-Y\|_2}{\sigma_{\min}(\Omega)},
 \qquad \beta=\sigma_{\min}(O_-).
 \label{eq:v6errorlevel}
\end{equation}
If $\delta<\beta$, the full-rank least-squares reconstruction satisfies
\begin{equation}
 \|\widehat F-F\|_2\le
 \frac{(1+\|F\|_2)\delta}{\beta-\delta}.
 \label{eq:v6stability}
\end{equation}
Enforcing the exact first block $L$ preserves the bound.
\end{proposition}

Here $\delta$ bounds history reconstruction error, while $\beta=\min_{\|h\|_2=1}\|O_-h\|_2$ is the weakest sensor-history response to a unit source-state displacement over the shift lags; what matters is a weakest response large relative to reconstruction error. Weak exposure or poorly separated leadfield columns can already destabilize the preceding labeling step (Appendix~\ref{app:v6stability}). The bound assumes correct $L,K$ and labels; it does not cover model mismatch or give an unconditional success probability.

\section{Experiments}
\label{sec:v6experiments}

The experiments follow the argument of the paper: whether unknown local changes supply information that baseline means lack and whether it is usable at finite budgets (Section~\ref{sec:v6partial}); how reliably direct reconstruction uses that information when baseline responses would already suffice (Section~\ref{sec:v6spectral}); and what the guarantees do not cover (Section~\ref{sec:v6controls}). Known simulated $F,z$ permit direct evaluation of dynamics and source recovery that sensor prediction alone cannot certify.

\begin{figure}[t]
 \centering
 \includegraphics[width=0.74\linewidth]{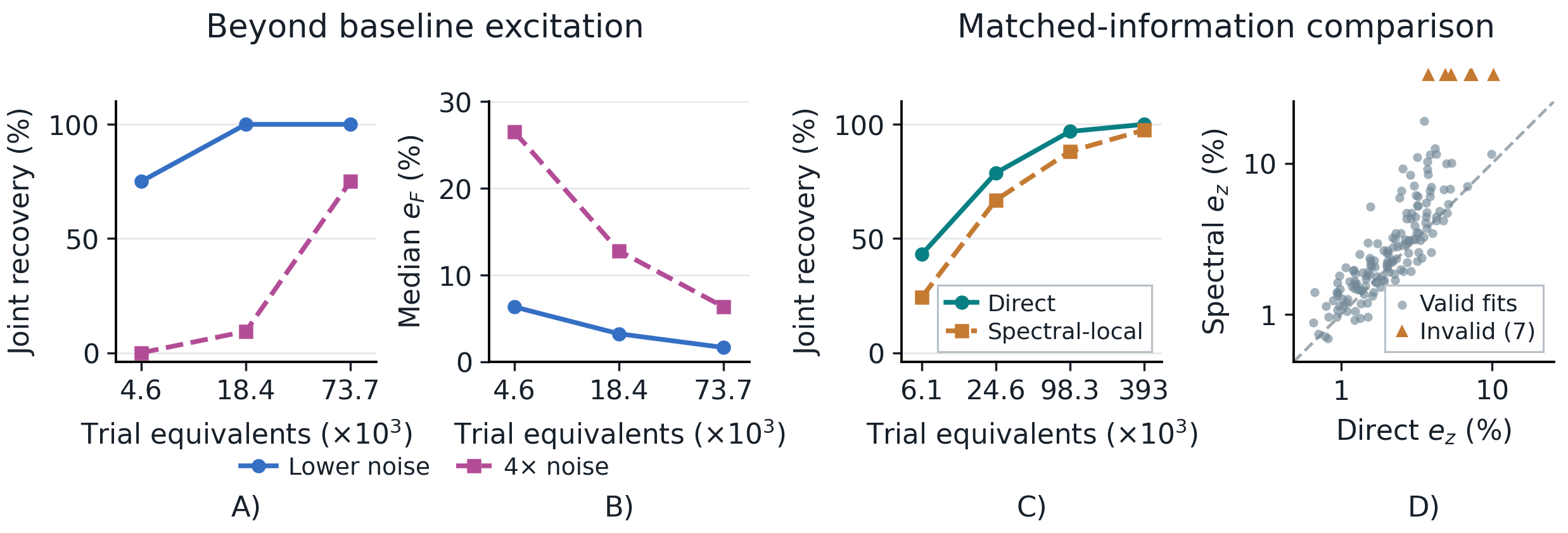}
 \caption{Recovery from simulated EEG. (A--B)~Direct with 4/12 baseline-reachable sources (32 anatomy--seed cases): joint recovery and median $e_F$ at sensor/process noise 0.01/0.002 or fourfold larger; paired Gaussian draws make the higher-noise 73,728 and lower-noise 4,608 cells coincide ($4/\sqrt{192}=1/\sqrt{12}$). (C--D)~Separate full-reachability comparison (160 cases): Direct versus Spectral-local under matched information. Panel D pairs $e_z$ at 98,304 trial equivalents; points above the diagonal favor Direct; seven invalid Spectral-local fits occupy a nonmetric upper strip. Joint recovery requires $e_F,e_z\leq10\%$.}
 \label{fig:v6recovery-overview}
\end{figure}

\subsection{Data and evaluation}

Four Localize-MI anatomies derived from magnetic resonance imaging (MRI) supply leadfields \citep{mikulan2020}. Principal studies use $q=12$ fixed cortical-normal sources, nine electrodes and $m=8$ independent sensor contrasts; columns are normalized once to define source units. Dynamics, interventions, noise and EEG are simulated with four coordinate initialization patterns, eight unknown targets and $T=s=6$; recovery is evaluated in anatomical coordinates without alignment to truth. A \emph{trial equivalent} is one initialized recording episode; we sample Gaussian episode averages with the covariance of the represented repeats, and half the budget is baseline (Appendix~\ref{app:v6acquisition}). Estimators receive $L,K,q$ and measured responses only, never true targets or coefficients. Evaluation reconstructs 32 new distributed initial states from their entire eight-step noisy sensor histories. With $Z,\widehat Z$ stacking true and reconstructed trajectories,
\begin{equation}
 e_F=\frac{\Frob{\widehat F-F}}{\Frob F},
 \qquad e_z=\frac{\Frob{\widehat Z-Z}}{\Frob Z},
 \label{eq:v6metrics}
\end{equation}
with $\|\cdot\|_{\mathrm F}$ the Frobenius norm. Joint success requires $e_F,e_z\le0.10$. Invalid estimates have infinite errors and count as failures; all cases remain in reported success rates. Full protocols, trial accounting and complete results are in Appendices~\ref{app:v6experiments}--\ref{app:v6injective}.

\subsection{Recovery beyond baseline reachability}
\label{sec:v6partial}

We test whether unknown local changes enable recovery when baseline excitation remains rank-deficient at every horizon. In twelve-source systems, baseline initializations explore only four coordinates, while local changes can expose the remaining eight. To preserve this missing information at every baseline horizon, we construct
\begin{equation}
 F=\begin{bmatrix}A&C\\0&B\end{bmatrix},\quad
 K=\begin{bmatrix}I_4\\0\end{bmatrix},\quad
 F^\tau K=\begin{bmatrix}A^\tau\\0\end{bmatrix}\quad(\tau\ge0),
 \label{eq:v6partial-block}
\end{equation}
where $A\in\R^{4\times4}$, $B\in\R^{8\times8}$ and $C\in\R^{4\times8}$. Baseline means reach four coordinates at every horizon, so $\rank R_s=\rank H_0=4$. Eight unknown general row changes target the remaining coordinates with $v_e^\top K\ne0$, completing $\Omega$; all systems have $\rank O_-=12$. The estimator receives neither the block structure nor the targets. Eight fresh systems per anatomy give 32 cases (Appendix~\ref{app:v6partial}).

\begin{table}[t]
\centering
\small
\setlength{\tabcolsep}{4pt}
\caption{Recovery beyond baseline reachability. Twelve-source systems whose baseline means reach four coordinates at every horizon; 32 anatomy--seed cases per row; errors in percent; $N$ counts trial equivalents. Top: Direct with all eight modes exposed (noise rows are paired acquisitions). Bottom: controls at 18,432 trial equivalents, lower noise; ``coverage'' is $\rank[K,E_J]$ over targets with nonzero population contrasts. Square-leadfield rows use fresh systems from the same generator with $m=q=12$ (Appendix~\ref{app:v6injective}).}
\label{tab:v6partial-main}
\begin{tabular}{@{}lrrrrr@{}}
\toprule
\multicolumn{6}{@{}l}{\textit{Direct, all eight modes exposed}}\\
Sensor/process SD & $N$ & Success & med.\ $e_F$ & med.\ $e_z$ & p90 $e_z$\\
\midrule
0.01/0.002 & 4,608  & 24/32 & 6.33 & 8.08 & 11.85\\
0.01/0.002 & 18,432 & \textbf{32/32} & 3.24 & 3.97 & 6.02\\
0.01/0.002 & 73,728 & \textbf{32/32} & 1.66 & 1.99 & 3.04\\
0.04/0.008 & 4,608  & 0/32 & 26.50 & 29.52 & 40.88\\
0.04/0.008 & 18,432 & 3/32 & 12.79 & 15.34 & 22.50\\
0.04/0.008 & 73,728 & 24/32 & 6.33 & 8.08 & 11.85\\
\midrule
\multicolumn{6}{@{}l}{\textit{Controls}}\\
Estimator / condition & Coverage & Valid & Success & med.\ $e_F$ & med.\ $e_z$\\
\midrule
Direct, one mode unexposed & 11 & 1/32 & 0/32 & $\infty$ & $\infty$\\
Direct, all modes unexposed & 4 & 0/32 & 0/32 & $\infty$ & $\infty$\\
Spectral-local (baseline realization) & 12 & 0/32 & 0/32 & $\infty$ & $\infty$\\
Baseline-only OLS, square $L$, entire budget & -- & 32/32 & 0/32 & 704.41 & 363.70\\
Baseline-only ridge--GCV, square $L$, entire budget & -- & 32/32 & 0/32 & 82.08 & 199.86\\
Direct, square $L$, same systems & 12 & 32/32 & \textbf{32/32} & 1.62 & 1.78\\
\bottomrule
\end{tabular}
\end{table}

\paragraph{Local changes recover what baseline cannot.}
At lower noise, direct successes are 24/32, 32/32 and 32/32 across increasing budgets (Table~\ref{tab:v6partial-main}; Figure~\ref{fig:v6recovery-overview}A--B). At 18,432 trial equivalents, median $e_F,e_z$ are 3.24\% and 3.97\%: all twelve-source dynamics are recovered despite rank-four baseline excitation, with every target labelled correctly. At fourfold noise, success is 0/32, 3/32 and 24/32: identification is exact, but finite-budget recovery requires precision, which repetition supplies without changing baseline reachability.

\paragraph{The gain is informational, not numerical.}
Four controls on the same regime show that no estimator can replace the exposed changes. First, exposure is necessary: keeping every change nonzero but removing its action on the tested states ($v_e^\top R_s=0$) for one mode reduces coverage to 11 and yields 0/32 successes at every budget and noise level; removing it for all modes yields 0/32 with no valid fit. Increasing repeats cannot make a zero population contrast nonzero. Second, baseline realization fails: Spectral-local, which factors baseline responses into an order-$q$ latent model before using contrasts, returns no valid reconstruction in any cell, because it must infer twelve coordinates from a rank-four baseline Hankel matrix. Third, an invertible readout does not help: with a square leadfield, inverting each sensor block and regressing the dynamics from baseline responses gives 0/32 with ordinary least squares (OLS) or ridge regression tuned by generalized cross-validation (GCV), even with the entire budget acquired as baseline, whereas Direct succeeds in 32/32 on the same systems. Fourth, the ambiguity is real: for every unexposed control system, an independent validator constructs alternative dynamics $F_S=SFS^{-1}$ preserving $L$, $K$ and all included mean response blocks; across 64 such witnesses, mean responses agree to $1.12\times10^{-16}$ while the dynamics differ by at least 2.13\%. The failures of baseline-only estimation therefore reflect missing information, not failed estimation (Appendices~\ref{app:v6partial} and~\ref{app:v6injective}). Coverage is sufficient, not minimal (Appendix~\ref{app:v6injective-fewer}).

\subsection{Reliability when baseline information suffices}
\label{sec:v6spectral}

With full baseline reachability, both routes have sufficient population information, and the comparison isolates estimator reliability. Spectral-local fits an order-$q$ latent model to baseline responses by singular value decomposition, then uses intervention contrasts to map it to anatomical coordinates (``spectral'' means matrix factorization, not EEG frequency analysis). The primary confirmation crosses 40 fresh seed blocks with four anatomies, using unknown row suppression in $[0.2,0.6]$: 160 cases per budget (Figure~\ref{fig:v6recovery-overview}C--D; Table~\ref{tab:v6matched-reliability}).

\begin{table}[t]
\centering
\setlength{\tabcolsep}{4pt}
\caption{Matched-information recovery under full baseline reachability. Methods share simulated responses within rows; 160 anatomy--seed cases per row at 98,304 trial equivalents. p90: empirical 90th percentile of $e_z$, including invalid fits as $\infty$.}
\label{tab:v6matched-reliability}
\begin{tabular}{@{}lrrrr@{}}
\toprule
 & \multicolumn{2}{c}{Joint success $\uparrow$}
     & \multicolumn{2}{c}{p90 $e_z$ (\%) $\downarrow$}\\
\cmidrule(lr){2-3}\cmidrule(l){4-5}
Intervention law & Direct & Spectral-local & Direct & Spectral-local\\
\midrule
Row suppression (primary) &\textbf{155/160}&141/160&\textbf{4.28}&9.20\\
General row change (control) &\textbf{157/160}&139/160&\textbf{4.16}&9.99\\
\bottomrule
\end{tabular}
\end{table}

Direct improves joint success by a paired 8.75 percentage points (approximate 95\% seed-cluster bootstrap interval $[3.125,15.0]$). Median source errors remain close (2.20\% versus 2.36\%); the gain is fewer large errors (lower p90) and no invalid fits versus seven for Spectral-local (Figure~\ref{fig:v6recovery-overview}D). Recovery does not require proportional suppression: norm-matched arbitrary receiving-row changes give the same picture (Table~\ref{tab:v6matched-reliability}; Appendix~\ref{app:v6arbitrary}). The advantage persists against inverse-then-regress estimation using exactly the same interventional responses with a square leadfield: OLS and ridge--GCV give 0/32 at 18,432 trial equivalents where Direct gives 32/32, because they invert each sensor block before exploiting the rank-one contrast structure; with sixteen measurements and a fourfold budget both routes succeed (Appendix~\ref{app:v6injective}). Direct reconstruction also makes iterative fitting usable under partial excitation: in an exploratory known-leadfield Gaussian state-space comparison using expectation-maximization (EM), initializing at the direct estimate yields higher training likelihood than every non-direct start and 32/32 convergences, versus 3/32 for likelihood-selected non-direct fits at a 500-iteration cap (Appendix~\ref{app:v6em}). Fixed-budget recovery degrades with system size for both methods (Appendix~\ref{app:v6fresh}); these results support reliability over the comparator, not scalability at fixed cost.

\subsection{What the guarantees do not cover}
\label{sec:v6controls}
\label{sec:v6injective}
\label{sec:v6physical-readout}
\label{sec:v6recorded-main}

The recovery guarantee is conditional on locality, calibrated $L,K$ and correct labels, and the stability bound on correct $L$. Three separate prespecified stress tests delimit these conditions.

\paragraph{Locality.}
Adding a neighboring-row change at 0.1 of the primary coefficient reduces Direct to 65/80 successes at 98,304 trial equivalents and 69/80 at four times that budget: averaging does not remove the bias from nonlocal changes (Appendix~\ref{app:v6experiments}).

\paragraph{Forward model.}
Sixteen simulated systems are paired across five assumed skull-conductivity settings in physical head models, retaining source locations, orientations and units. With the nominal $L$ supplied to the estimator and exact responses, every target label is correct, yet nonnominal settings produce median $e_F$ of 11.17--27.52\%; supplying the generating leadfield recovers all 80 paired cases to numerical precision. Correct labels therefore do not imply correct dynamics, and this bias persists without sampling noise, outside the correct-$L$ stability bound (Appendix~\ref{app:v6physical-readout}).

\paragraph{Recorded EEG.}
Real data can currently test only the spatial half of the calibration. In Localize-MI EEG from the same four participants, within a $-2$ to $+2$ ms window dominated by injected-current artifact, a spatial direction learned at one intensity predicts responses at another on disjoint test trials across eleven sites (median score 0.949 versus 0.179 for a prespecified other-site direction), and trial averaging makes dipole fits more repeatable (agreement with the full-pool fit 66.88\% to 92.50\%) without moving them closer to the contact midpoint (13.68 mm). Spatial transfer and repeatability do not establish accurate anatomical calibration, and these analyses recover no neural $F$, validate no neural $K$, and implement no mechanism change (Appendix~\ref{app:v6recorded}).

\section{Discussion}
\label{sec:v6discussion}
\label{sec:v6applicability}

We show how unknown local mechanism changes can reveal source dynamics that controlled baseline mean responses leave undetermined. The central insight is to identify an anatomically calibrated source-response history rather than every coefficient of the intervention; combined with known initialization responses, these histories recover directed interactions between anatomically defined sources beyond baseline reachability. The experiments separate this information gain from the estimator: partially excited systems are recovered where baseline realization, invertible readout and unexposed changes all fail, and explicit alternative dynamics reproduce every baseline mean; where baseline information suffices, direct reconstruction remains the more reliable route and initializes likelihood fitting. The stress tests state the price of the guarantee: noise, nonlocal changes and forward-model error each degrade recovery even when source labels are correct, so labels and histories must be validated separately.

These findings motivate mapping interactions in controlled neural-circuit experiments, including dynamic-clamp and optogenetic settings \citep{sharp1993dynamic,wilson2012division}, and potentially EEG-based perturbation studies; biological translation requires calibrated observation and initialization models and validation of the single-row, single-transition protocol (Appendix~\ref{app:v6applicability}).
\label{maintextend}

\subsection*{Ethics and data use}
The study combines existing anatomical operators from Localize-MI \citep{mikulan2020} with computational simulations and secondary analyses of the dataset's recorded scalp EEG during intracerebral current injection. No new human recordings or stimulation procedures were acquired. Data reuse and redistribution must respect the dataset's CC BY-NC-SA 4.0 licence and applicable privacy requirements. The mathematical intervention protocol is not a clinically validated stimulation or diagnostic procedure.

\subsection*{Reproducibility}
Appendices~\ref{app:v6theory} and~\ref{app:v6injective} provide the main mathematical arguments. Appendices~\ref{app:v6experiments}, \ref{app:v6partial} and~\ref{app:v6injective} specify acquisition, estimators, evaluation and complete results for the central simulation studies. Appendix~\ref{app:v6em} documents the exploratory known-leadfield EM comparison, including initializations, stopping rules and numerical checks. Appendices~\ref{app:v6signed} and~\ref{app:v6gain-boundary} give additional-assumption controls, including their derivations and failure cases; Appendix~\ref{app:v6recorded} describes recorded-current EEG diagnostics and their distinct scope. Appendix~\ref{app:v6physical-readout} documents the physical forward-model sensitivity study. Appendix~\ref{app:v6applicability} expands the experimental interpretation and related work. Experimental protocols, numerical verification records and checksums accompany the research package.

Code for this work is available at:\\
\url{https://github.com/AyanaMussabayeva/source_localization}.

\bibliography{references}
\bibliographystyle{plainnat}
\clearpage

\appendix
\etocdepthtag.toc{appendix}
\begingroup
\hypersetup{linktoc=all}
\etocsettagdepth{main}{none}
\etocsettagdepth{appendix}{subsection}
\etocsetnexttocdepth{subsection}
\etocsettocstyle{%
  \pdfbookmark[0]{Appendix contents}{appendix-contents}%
  \section*{Appendix contents}%
  \noindent Section titles and page numbers link to the corresponding
  appendix material.\par\medskip
}{}
\tableofcontents
\endgroup

\section{Identification assumptions and proofs}
\label{app:v6theory}

\paragraph{Notation conventions.}
Source indices are one-based. A colon selects all entries on that axis:
$L_{:,j}$ is a column, $F_{j,:}$ a row, and $(O_T)_{:,U}$ selects columns
indexed by a source set $U$. A hat denotes an estimate. The norm $\|\cdot\|_2$
is Euclidean for vectors and spectral for matrices; $\|\cdot\|_{\rm F}$ is
the Frobenius norm. The transpose is $\top$, whereas $\dagger$ is the
Moore--Penrose pseudoinverse, not a transpose or a new model parameter.
It gives the minimum-norm least-squares solution; full row rank of $\Omega$
gives $\Omega\Omega^\dagger=I_q$, and full column rank of $O_-$ gives
$O_-^\dagger O_-=I_q$.

\begin{table}[ht]
\centering\small
\caption{Core notation. $|J|$ is the number of distinct exposed targets.}
\label{tab:v6notation}
\begin{tabular}{@{}p{0.26\linewidth}p{0.22\linewidth}p{0.44\linewidth}@{}}
\toprule
Symbol & Dimension & Meaning\\
\midrule
$z_{\tau},x_{\tau},u$ & $q,m,r$ & Source state, sensor observation, controlled input\\
$T,s$ & Positive integers & Output-history and insertion-time horizons\\
$t,\tau_{\mathrm{int}},\tau$ & Nonnegative integers & Output lag $t=0,\ldots,T-1$; insertion time $\tau_{\mathrm{int}}=0,\ldots,s-1$; absolute episode time $\tau$\\
$F$ & $q\times q$ & Baseline transition matrix\\
$L,K$ & $m\times q,q\times r$ & Leadfield and initialization map\\
$D_e=e_{j_e}v_e^\top$ & $q\times q$ & One-transition change to receiving row $j_e$\\
$O_T,R_s$ & $mT\times q,q\times rs$ & Temporal observation and input histories\\
$H_0,H_+,H^{[e]},\Delta H_e$ & $mT\times rs$ & Baseline, shifted, inserted, contrast blocks\\
$o_j,\widehat u_e,\psi_e$ & $mT,mT,rs$ & Physical column, unit direction, right profile\\
$E_J$ & $q\times|J|$ & Target coordinate vectors\\
$\Omega=[K,E_J]$ & $q\times(r+|J|)$ & Anchoring matrix\\
$Y=O_T\Omega$ & $mT\times(r+|J|)$ & Observed and calibrated columns\\
$O_-,O_+$ & $m(T-1)\times q$ & Observation stacks before and after a time shift\\
$\Phi,\Phi_0$ & $q\times q$ & Candidate and true maps from latent to anatomical coordinates\\
$S,\Delta S,\Delta\Phi$ & $q\times q$ & Coordinate gauge, correction $\Delta S=S-I_q$, map difference $\Phi-\Phi_0$\\
$\sigma_e,\alpha_j,\mu_j,\beta$ & Scalars & Exposure, visibility, label angle, shift conditioning\\
$\delta$ & Scalar & Column error after coverage amplification\\
$e_F,e_z$ & Scalars & Relative dynamics and source-trajectory reconstruction errors\\
\bottomrule
\end{tabular}
\end{table}

\subsection{Population responses and anatomical identification}
\label{app:v6setup}

Let the anatomically indexed state obey
\begin{equation}
 z_{\tau+1}=Fz_{\tau}+\epsilon_{\tau},
 \qquad x_{\tau}=Lz_{\tau}+\nu_{\tau},
 \label{app:v6eq:model}
\end{equation}
where \(F\in\R^{q\times q}\), the effective observation map
\(L\in\R^{m\times q}\), and the known initialization map \(K\in\R^{q\times r}\)
specifies the initial conditional mean $\mathbb E[z_0\mid u]=Ku$.
There is no subsequent baseline additive input. Innovations and sensor noise
have zero conditional mean under the exogenous input and insertion schedule.
Write \(\rho=\rank L\) and \(\kappa=\rank K\). The source locations,
orientations, ordering, and units are fixed by \(L\); changing a column
normalization changes the physical coordinate convention.

For positive integers \(T,s\), define
\begin{equation}
 O_T=
 \begin{bmatrix}L\\LF\\ \vdots\\LF^{T-1}\end{bmatrix},
 \qquad
 R_s=[K,FK,\ldots,F^{s-1}K].
 \label{app:v6eq:or}
\end{equation}
An intervention \(e\) acts for one transition and changes one incoming row,
\begin{equation}
 D_e=e_{j_e}v_e^\top,
 \label{app:v6eq:rowchange}
\end{equation}
where both \(j_e\) and \(v_e\in\R^q\) are unknown.

\paragraph{Matched response histories.}
Choose an output horizon $T\ge2$ and a pre-insertion horizon $s\ge1$.
Insert mode $e$ at transition
$\tau_{\mathrm{int}}\to\tau_{\mathrm{int}}+1$, with
$\tau_{\mathrm{int}}=0,\ldots,s-1$. Relative output time
$t=0,\ldots,T-1$ corresponds to episode time
$\tau=\tau_{\mathrm{int}}+1+t$. Thus
\begin{equation}
 \bar x^{(e)}_{t,\tau_{\mathrm{int}}}(u)
 :=\mathbb E[x_{\tau_{\mathrm{int}}+1+t}\mid u,e\text{ inserted at }\tau_{\mathrm{int}}]
 =LF^t(F+D_e)F^{\tau_{\mathrm{int}}} Ku.
 \label{eq:v6episode}
\end{equation}
In particular, $t=0$ is the first post-transition observation, not the
observation at intervention onset.

\paragraph{Building response matrices from EEG.}
Each experimental condition yields a multichannel EEG time course. For each
intervention mode, columns of its response matrix index initialization and
intervention-time conditions, while rows index sensors and times after the
changed transition. Let $u^{(i)}\in\R^r$ be the $i$th unit input, so
$Ku^{(i)}=K_{:,i}$, for $i=1,\ldots,r$. Fix one input $i$ and insertion
time $\tau_{\mathrm{int}}$. They select column $r\tau_{\mathrm{int}}+i$
in both matrices: $R_s$ contains the pre-intervention state
$F^{\tau_{\mathrm{int}}}K_{:,i}$, and $H^{[e]}$ contains its stacked mean
EEG history. Left multiplication by $O_T(F+D_e)$ maps one to the other:
\begingroup
\definecolor{responseCondition}{HTML}{F7DFC4}
\setlength{\arraycolsep}{3pt}
\renewcommand{\arraystretch}{1.12}
\begin{equation}
 R_s=\left[\begin{array}{ccc}
 \cdots & \cellcolor{responseCondition}F^{\tau_{\mathrm{int}}}K_{:,i} & \cdots
 \end{array}\right]
 \xrightarrow[\text{column }r\tau_{\mathrm{int}}+i]{\ O_T(F+D_e)\ }
 H^{[e]}=\left[\begin{array}{ccc}
 \cdots & \cellcolor{responseCondition}\bar x^{(e)}_{0,\tau_{\mathrm{int}}}(u^{(i)}) & \cdots\\
 \hline
 \vdots & \cellcolor{responseCondition}\vdots & \vdots\\
 \hline
 \cdots & \cellcolor{responseCondition}\bar x^{(e)}_{T-1,\tau_{\mathrm{int}}}(u^{(i)}) & \cdots
 \end{array}\right].
 \label{eq:v6observed-column}
\end{equation}
\endgroup
Shading follows one experimental condition, not one source. Each displayed
$\bar x$ is an $m$-entry sensor vector; horizontal rules separate output-lag
blocks. Within the selected column, $i$ and $\tau_{\mathrm{int}}$ stay fixed
while $t$ varies from $0$ to $T-1$. Thus the column has $mT$ entries, and
the $rs$ conditions form $H^{[e]}\in\R^{mT\times rs}$, with $i$ varying
fastest. For example, with two initialization inputs ($r=2$), $i=2$ and
$\tau_{\mathrm{int}}=0$ select column 2; this is a column number, not an
intervention time. Trial averages give $\widehat H^{[e]}$ from EEG without
knowing $F$ or hidden $z_\tau$. The matched baseline $H_+$ uses the same
inputs and times $\tau_{\mathrm{int}}+1+t$ without intervention; $H_0$ uses
times $\tau_{\mathrm{int}}+t$, one sample earlier.

For fixed $t$ and $\tau_{\mathrm{int}}$, the block
$H^{[e]}[t,\tau_{\mathrm{int}}]$ contains all $m$ sensors and all $r$ inputs:
it selects rows $mt+1,\ldots,m(t+1)$ and columns
$r\tau_{\mathrm{int}}+1,\ldots,r(\tau_{\mathrm{int}}+1)$.
For $t=0,\ldots,T-1$ and $\tau_{\mathrm{int}}=0,\ldots,s-1$, the matched
population blocks are
\begin{align}
 H_+[t,\tau_{\mathrm{int}}]&=LF^tF F^{\tau_{\mathrm{int}}} K,\nonumber\\
 H^{[e]}[t,\tau_{\mathrm{int}}]&=LF^t(F+D_e)F^{\tau_{\mathrm{int}}} K.
 \label{app:v6eq:blocks}
\end{align}
Equivalently, after stacking the block indices,
\(H_+=O_TFR_s\). Each block is an $m\times r$ linear map from the
controlled input to a conditional mean, not to a noisy single trial:
$\bar x^{(e)}_{t,\tau_{\mathrm{int}}}(u)=H^{[e]}[t,\tau_{\mathrm{int}}]u$.
The complete matrices $H_0,H_+,H^{[e]}$ have size $mT\times rs$ and obey
Eq.~\eqref{eq:v6blocks}. Unlike $O_T$, they contain responses for tested
experimental conditions rather than unit-source histories. Baseline lags
reused across blocks share the same estimated response. The matched schedule
makes $\Delta H_e=H^{[e]}-H_+$ observable; arbitrary stationary recording
conditions need not supply this contrast.

The model assumes known \(q,L,K\), identifiable population responses in
\eqref{app:v6eq:blocks}, and common state coordinates and baseline \(F\)
across conditions. The intervention preserves the state at insertion,
changes one receiving row for one transition, and leaves the first
post-transition output available. The results identify response means;
trajectory distributions and innovation covariances require an additional
noise model.

There are two distinct identification routes. The direct route uses
\eqref{app:v6eq:blocks} and does not assume
\(\rank R_s=q\). The shared-realization route instead assumes that a
joint identification procedure has returned
\begin{equation}
 A=\Phi_0^{-1}F\Phi_0,\qquad C=L\Phi_0,\qquad B=\Phi_0^{-1}K,
 \qquad \Delta_e=\Phi_0^{-1}D_e\Phi_0
 \label{app:v6eq:shared}
\end{equation}
for one invertible \(\Phi_0\) common to every mode. Independently fitted regime
realizations do not provide \eqref{app:v6eq:shared}. A sufficient classical
construction is a rank-\(q\) factorization of
\(H_0=O_TR_s\) when both factors have rank \(q\)
\citep{ho1966}; this sufficient construction is not an assumption of the
direct route.

\subsection{Anatomical column identification}
\label{app:v6contrast}

\begin{lemma}[Observable contrast factorization]
\label{app:v6lem:contrast}
For the row change \eqref{app:v6eq:rowchange},
\begin{equation}
 \Delta H_e:=H^{[e]}-H_+
 =(O_Te_{j_e})(v_e^\top R_s).
 \label{app:v6eq:contrast}
\end{equation}
If \(v_e^\top R_s\neq0\) and \(L_{:,j_e}\neq0\), then \(\Delta H_e\) has rank one and its
left singular subspace is
\(\operatorname{span}(O_Te_{j_e})\). Its first \(m\) entries are
proportional to \(L_{:,j_e}\).
\end{lemma}
\begin{proof}
For each block pair,
\[
 H^{[e]}[t,\tau_{\mathrm{int}}]-H_+[t,\tau_{\mathrm{int}}]
 =LF^te_{j_e}v_e^\top F^{\tau_{\mathrm{int}}} K.
\]
Stacking \(t\) and \(\tau_{\mathrm{int}}\) gives \eqref{app:v6eq:contrast}. A nonzero outer
product has the asserted rank and column space, and the first block of
\(O_Te_{j_e}\) is \(Le_{j_e}=L_{:,j_e}\).
\end{proof}

Suppose every candidate \(L_{:,j}\) is nonzero and the projective map
\(j\mapsto\operatorname{span}(L_{:,j})\) is injective. The first block then labels
\(j_e\). For the leading unit left singular vector $\widehat u_e$ of
$\widehat{\Delta H}_e$, let $\widehat u_{e,\mathrm{top}}$ be its first $m$ entries.
A zero first block is invalid. Otherwise, at inferred target $j$, the
least-squares scale and calibrated physical column are
\[
 \gamma_e=\frac{\widehat u_{e,\mathrm{top}}^\top L_{:,j}}
 {\|\widehat u_{e,\mathrm{top}}\|_2^2},\qquad
 \widehat o_j=\gamma_e\widehat u_e.
\]
With exact exposed responses and correct labels, this equals $o_j=O_Te_j$.
The arbitrary singular-vector sign is absorbed by $\gamma_e$.
Proportional leadfield columns require a joint
assignment test; if more than one assignment satisfies the completion
constraints, the target is not identified. Repeated interventions at one
target improve neither the number of distinct physical columns nor the exact
coverage rank, although they may improve precision.

The exposure condition \(v_e^\top R_s\neq0\) does not require
\(\rank R_s=q\): an arbitrary row change can map a reachable direction
into a source that baseline inputs never excite, revealing its observability
column.

\subsection{Direct completion without baseline controllability}
\label{app:v6completion}

Let \(J\) be the set of distinct targets obtained from exposed contrasts, let
\(E_J=[e_j]_{j\in J}\), and define
\begin{equation}
 \Omega=[K,E_J],\qquad
 Y=[O_TK,O_TE_J]=O_T\Omega.
 \label{app:v6eq:completion-data}
\end{equation}
The first term in \(Y\) is a baseline response block; the second consists of
the scaled contrast columns from Lemma~\ref{app:v6lem:contrast}.

\begin{theorem}[Direct physical completion]
\label{app:v6thm:completion}
Assume \(T\geq2\), all columns in \eqref{app:v6eq:completion-data} are correctly
labeled and scaled, and
\begin{equation}
 \rank\Omega=q.
 \label{app:v6eq:coverage}
\end{equation}
Then
\begin{equation}
 O_T=Y\Omega^\dagger.
 \label{app:v6eq:ocomplete}
\end{equation}
Writing
\(O_-=[L^\top,(LF)^\top,\ldots,(LF^{T-2})^\top]^\top\)
and letting \(O_+\) be the same stack shifted by one power of \(F\),
if \(\rank O_-=q\), then
\begin{equation}
 F=O_-^\dagger O_+.
 \label{app:v6eq:fshift}
\end{equation}
No full-row-rank assumption on \(R_s\) is needed.
\end{theorem}
\begin{proof}
Condition \eqref{app:v6eq:coverage} gives \(\Omega\Omega^\dagger=I_q\), so
\(Y\Omega^\dagger=O_T\Omega\Omega^\dagger=O_T\). The block-shift identity
is \(O_+=O_-F\). Full column rank of \(O_-\) permits
left pseudoinversion and gives \eqref{app:v6eq:fshift}.
\end{proof}

Since each distinct coordinate target adds at most one rank to \(\Omega\), at least
\(q-\kappa\) targets are required for \eqref{app:v6eq:coverage}. A coordinate
set complementing \(\operatorname{im}K\) attains this count, but an arbitrary
set of that size need not. This is the count required by direct column
completion, not a universal intervention lower bound.
Appendix~\ref{app:v6injective-propagation} gives another route with fewer
targets even under unrestricted row locality when the known anatomy is
injective; additional symmetry, sparsity or a gain law can also change
the information available.

\paragraph{Exact noncontrollable example.}
Let \(q=7\), choose distinct
\(0<\lambda_1<\cdots<\lambda_7<1\), and set
\[
 F=\diag(\lambda_1,\ldots,\lambda_7),\qquad
 K=[e_1,e_2],\qquad
 L_{:,j}=(1,\lambda_j)^\top.
\]
Then \(\rank R_s=2\) for every \(s\), and the baseline Hankel matrix has rank
two. Nevertheless, for \(T\geq8\), the submatrix obtained by taking the first
sensor row from the first seven block rows of \(O_-\) is the Vandermonde
matrix \([\lambda_j^\tau]_{\tau=0,\ldots,6;\,j=1,\ldots,7}\); hence
\(\rank O_-=7\). The columns \(L_{:,j}\) are projectively distinct. For targets
\(J=\{3,\ldots,7\}\), choose arbitrary row changes
\(D_j=e_je_1^\top\). Each is exposed because \(e_1^\top R_s\neq0\), and
\([K,E_J]=I_7\), so Theorem~\ref{app:v6thm:completion} recovers \(F\) although
the baseline is noncontrollable.

\subsection{Complete calculation for the three-source example}
\label{app:v6worked-example}

Consider three sources, two sensor measurements and one initialization
pattern:
\begin{equation}
 L=\begin{bmatrix}1&0&1\\0&1&1\end{bmatrix},\qquad
 K=e_1,\qquad
 F=\begin{bmatrix}\tfrac12&1&0\\0&0&1\\0&0&0\end{bmatrix}.
 \label{eq:v6toy-system}
\end{equation}
The reconstruction receives known inputs $(L,K)$ and measured conditional
means $(H_0,H_+,H^{[1]},H^{[2]})$ and recovers
$(j_1,j_2,O_3,F)$. The generating transition $F$ and intervention coefficients
are displayed only to check the calculation and are not supplied to the
estimator. This example uses exact conditional means. The leadfield has rank two,
but its three columns are nonzero and pairwise nonproportional.

\paragraph{Baseline ambiguity.}
Baseline initialization gives $F^\tau K=2^{-\tau}e_1$: source 1 decays
while sources 2 and 3 remain unexcited in the mean. Thus $\rank R_s=1$
at every horizon. Changing the influence of source 3 on source 2 gives
the distinct transition $\widetilde F=F+\tfrac12e_2e_3^\top$.
It also satisfies $\widetilde F^\tau K=2^{-\tau}e_1$, so
\begin{equation}
 LF^\tau K=L\widetilde F^\tau K=2^{-\tau}L_{:,1}
 \qquad(\tau\ge0).
 \label{eq:v6toy-ambiguity}
\end{equation}
Longer baseline recordings or more repetitions cannot distinguish these
dynamics through their conditional means.

\paragraph{Measured contrasts and source histories.}
Set $r=1$, $u=1$, $T=3$, and $s=2$, with insertion times
$\tau_{\mathrm{int}}=0,1$. Let the two unknown modes have row changes
\begin{equation}
 D_1=e_2(2,b,c),\qquad D_2=e_3(-1,d,f),
 \label{eq:v6toy-changes}
\end{equation}
where $b,c,d,f$ are arbitrary. Mode indices 1 and 2 label experiments;
their source targets are $j_1=2$ and $j_2=3$. Both targets and all changed
coefficients are hidden from the estimator. Direct multiplication gives
\[
 O_3=\left[\begin{array}{ccc}
 1&0&1\\0&1&1\\\hline
 \tfrac12&1&0\\0&0&1\\\hline
 \tfrac14&\tfrac12&1\\0&0&0
 \end{array}\right]=[o_1,o_2,o_3],\qquad
 R_2=\begin{bmatrix}1&\tfrac12\\0&0\\0&0\end{bmatrix}.
\]
Columns of $R_2$ index insertion times, whereas columns of $O_3$ index
sources. Horizontal rules separate the two-sensor blocks at output lags
0, 1 and 2. The measured response matrices are
\begin{align*}
 H_0&=o_1[1,\tfrac12], & H_+&=o_1[\tfrac12,\tfrac14],\\
 H^{[1]}&=H_++o_2[2,1], &
 H^{[2]}&=H_++o_3[-1,-\tfrac12].
\end{align*}
These are sensor means obtained by applying the changed transition once
and then resuming $F$. The histories $o_2,o_3$ are not directly observed;
the equalities describe how the measured means factor.
Subtracting the matched baseline gives
\begin{equation}
 \Delta H_1=o_2[2,1],\qquad
 \Delta H_2=o_3[-1,-\tfrac12].
 \label{eq:v6toy-contrasts}
\end{equation}
Waiting one baseline step halves the exposed state and contrast amplitude,
but preserves the subsequent response shape. Because $R_2$ has zero
second and third rows, the arbitrary coefficients $b,c,d,f$ disappear
from $v_e^\top R_2$ and all these measured means. The histories can
therefore be identified even though these coefficients remain undetermined.

\paragraph{Labeling and calibration.}
For exact contrasts, one nonzero column suffices. The first sensor blocks
of the first contrast columns are $(0,2)^\top=2L_{:,2}$ for mode 1 and
$(-1,-1)^\top=-L_{:,3}$ for mode 2. Matching their directions to $L$
identifies sources 2 and 3. Dividing each \emph{entire} contrast column
by its signed factor, 2 or $-1$, recovers $o_2$ or $o_3$.
The same calculation with Algorithm~\ref{alg:v6direct} uses normalized
left singular vectors: $\|o_2\|_2=3/2$, $\|o_3\|_2=2$, so the vectors
are $\widehat u_1=\pm(2/3)o_2$ and $\widehat u_2=\pm o_3/2$.
Equation~\eqref{eq:v6calibration} gives $\gamma_1=\pm3/2$ and
$\gamma_2=\pm2$, with matching signs. Each product $\gamma_e\widehat u_e$
therefore recovers the correctly scaled history, regardless of the SVD sign.
Combining the histories with the first column of $H_0$, which equals $o_1$,
gives $\Omega=I_3$ and $Y=O_3$.

\paragraph{Recovering the transition.}
The first four rows of $O_3$ form $O_-$ and the last four form $O_+$;
the middle sensor block belongs to both, and $O_+=O_-F$.
The submatrix of $O_-$ on rows 1, 2 and 4 is
$\left[\begin{smallmatrix}1&0&1\\0&1&1\\0&0&1\end{smallmatrix}\right]$,
whose determinant is one. Thus $\rank O_-=3$ although $\rank L=2$:
time distinguishes source states that a single measurement cannot.
The shift equation has the unique solution $F$ in
Eq.~\eqref{eq:v6toy-system}, while $\rank R_s$ remains one.
The baseline-invisible alternative $\widetilde F$ has a different
source-3 history and is distinguished by the second intervention.
The singular values of $O_-$ are approximately
$2.01491804,1.18164239,0.89096944$, while all singular values of $\Omega$
are one. In the notation of Proposition~\ref{prop:v6stability},
$\sigma_{\min}(\Omega)=1$ and $\beta\approx0.891$ in the chosen units.
Exact means give $\delta=0$, so Eq.~\eqref{eq:v6stability} gives zero
reconstruction error. These values illustrate the bound; they are not
a universal threshold or a noisy-data guarantee.

Both the baseline ambiguity and the unresolved coefficients concern the
specified conditional means, not equality of full trial distributions or
what additional trial covariances could reveal. The example establishes
population identification without adding a finite-noise performance result.

\subsection{Coordinate ambiguity}
\label{app:v6gauge}

The constructive identification result gives a sufficient route to unique
recovery. What can remain unresolved when the available directions do not
fully cover the source space? We examine source-coordinate changes that
preserve the known anatomy, initialization patterns, and intervention
targets, and therefore leave the corresponding mean responses unchanged.

For an invertible coordinate change $S$, define the coordinate correction
$\Delta S:=S-I_q$. Preserving anatomy and the anchored directions requires
\begin{equation}
 L(\Delta S)=0,\qquad (\Delta S)\Omega=0.
 \label{eq:v6gauge}
\end{equation}
\begin{proposition}[An indistinguishable coordinate family]
\label{prop:v6gauge}
Let $J$ include the targets of every intervention mode under comparison.
Each invertible $S=I_q+\Delta S$ satisfying Eq.~\eqref{eq:v6gauge} gives
parameters
\begin{equation}
 F_S=SFS^{-1},\qquad D_{e,S}=e_{j_e}(v_e^\top S^{-1})
 \label{eq:v6alternatives}
\end{equation}
with the same $L,K$ and all the same mean responses to these modes. The
linear space of admissible corrections $\Delta S$ has dimension
\begin{equation}
 g_{\rm row}=(q-\rank L)(q-\rank\Omega).
 \label{eq:v6dimension}
\end{equation}
If $(F,\Omega)$ is controllable, i.e.,
$\rank[\Omega,F\Omega,\ldots,F^{q-1}\Omega]=q$, every nonidentity member
changes $F$.
\end{proposition}

$\Delta S$ maps unanchored directions into the sensor-invisible space.
Full coverage removes this family, explaining how interventions fix physical
coordinates. The result describes a coordinate family, not all models
compatible with finite responses; the argument below also treats special
dynamics. The coverage condition is sufficient, not a universal minimum on
intervention count.

\begin{proof}
Let $J$ contain the targets of every included intervention and let
$S=I_q+\Delta S$ be invertible with $L(\Delta S)=0$ and $(\Delta S)[K,E_J]=0$.
Then $LS=L$, $SK=K$, and $Se_j=e_j$ for $j\in J$.
Consequently $SD_eS^{-1}=e_{j_e}(v_e^\top S^{-1})$ remains in the same
row-local model. If $M_a$ is any included baseline or intervention transition,
then for any finite word of these transitions,
\[
 L(SM_{a_n}S^{-1})\cdots(SM_{a_1}S^{-1})K
 =LSM_{a_n}\cdots M_{a_1}S^{-1}K
 =LM_{a_n}\cdots M_{a_1}K.
\]
Thus all such mean responses agree, including the measured single-insertion
episodes. The admissible corrections $\Delta S$ are precisely linear maps from
$\R^q/\operatorname{im}[K,E_J]$ into $\ker L$, giving the dimension in
Eq.~\eqref{eq:v6dimension}. Invertibility holds in a neighborhood of $\Delta S=0$.
If $SFS^{-1}=F$, $S$ commutes with $F$ and fixes every column of the
controllability matrix of $(F,[K,E_J])$; if that matrix has rank $q$, $S=I_q$.
This proves the explicit family statement independently of a full-rank
baseline Hankel matrix. It does not exhaust all finite-window fits.
\end{proof}

\paragraph{Exhaustive coordinate maps within a shared realization.}
Under \eqref{app:v6eq:shared}, factor a nonzero
\(\Delta_e=h_{j_e}w_e^\top\). Lemma~\ref{app:v6lem:contrast}'s labeling and
scaling argument has the realization-space counterpart
\(h_j=\Phi_0^{-1}e_j\), hence \(\Phi_0h_j=e_j\). Let
\(H_J=[h_j]_{j\in J}\) and \(G=[B,H_J]\). Candidate anatomical maps solve
\begin{equation}
 L\Phi=C,\qquad \Phi B=K,\qquad \Phi H_J=E_J.
 \label{app:v6eq:anchor-system}
\end{equation}

\begin{theorem}[Affine completion and residual gauge]
\label{app:v6thm:gauge}
Assume an invertible true solution \(\Phi_0\) of
\eqref{app:v6eq:anchor-system}. Then
\begin{equation}
 \rank G=\rank[K,E_J]=:r_{\Omega},
 \label{app:v6eq:g-rank}
\end{equation}
and the affine space of all, not necessarily invertible, solutions has
dimension
\begin{equation}
 \boxed{(q-\rho)(q-r_{\Omega})}.
 \label{app:v6eq:gauge-dim}
\end{equation}
Every invertible solution near \(\Phi_0\) is \(\Phi=S\Phi_0\), where
\begin{equation}
 LS=L,\qquad SK=K,\qquad Se_j=e_j\quad(j\in J).
 \label{app:v6eq:row-gauge}
\end{equation}
Conversely, every invertible \(S\) satisfying
\eqref{app:v6eq:row-gauge} gives another solution. Thus the anatomical map is
unique exactly when \(\rho=q\) or \(r_{\Omega}=q\).

Uniqueness here concerns the coordinate map within the assumed common
realization. Injectivity does not by itself supply missing dynamical
excitation; Appendix~\ref{app:v6injective-baseline} characterizes the
baseline ambiguity that can remain even when $\rho=q$.

The associated transition is \(F_S=SFS^{-1}\). If the pair
\((F,[K,E_J])\) is controllable, every nonidentity admissible \(S\) changes
\(F\); under that additional condition, positive dimension in
\eqref{app:v6eq:gauge-dim} is also a necessity result for identifying \(F\).
Without it, coordinates may be ambiguous while a special \(F\) happens to be
invariant.
\end{theorem}
\begin{proof}
Multiplication by \(\Phi_0\) maps the columns of \(G\) bijectively to those of
\([K,E_J]\), proving \eqref{app:v6eq:g-rank}. If
\(\Delta\Phi=\Phi-\Phi_0\), the homogeneous equations are \(L\Delta\Phi=0\) and \((\Delta\Phi)G=0\).
They describe arbitrary linear maps from
\(\R^q/\operatorname{im}G\), of dimension \(q-r_{\Omega}\), into
\(\ker L\), of dimension \(q-\rho\), proving
\eqref{app:v6eq:gauge-dim}. Invertibility is open, so every nonzero
homogeneous direction supplies nearby invertible alternatives.

For an invertible solution set \(S=\Phi\Phi_0^{-1}\). Substitution in
\eqref{app:v6eq:anchor-system} gives \eqref{app:v6eq:row-gauge}; the reverse
substitution proves the converse. The transition represented by the same
latent \(A\) is \(\Phi A\Phi^{-1}=SFS^{-1}\). If this equals \(F\), then \(S\)
commutes with \(F\). Together with \(S[K,E_J]=[K,E_J]\), it fixes every column
of the controllability matrix of \((F,[K,E_J])\). Full row rank of that matrix
forces \(S=I_q\).
\end{proof}

The alternative intervention remains row local at the same labeled target:
\begin{equation}
 SD_eS^{-1}=e_{j_e}(v_e^\top S^{-1}).
 \label{app:v6eq:gauged-row}
\end{equation}
For noninjective $L$ and incomplete coverage, this family contains nonidentity
coordinate changes preserving every switched response word generated by the
same known inputs; they change $F$ under the controllability condition above.
This is a population mean-response equivalence. It is not
automatically a full-law equivalence when innovation covariances, state priors,
or other coordinate-specific noise structure are fixed rather than transformed
as nuisance parameters. Equation \eqref{app:v6eq:gauge-dim} is also
conditional on the target labels; discrete permutations can remain when
leadfield columns are projectively indistinguishable.

\subsection{Conditional deterministic stability}
\label{app:v6stability}

For an exposed intervention $e$ with target $j=j_e$, write
\(\Delta H_e=o_j\psi_e^\top\), with \(o_j=O_Te_j\) and
$\psi_e^\top=v_e^\top R_s$, and let
\(\widehat{\Delta H}_e=\Delta H_e+E_e\). Put
\(\sigma_e=\|o_j\|_2\|\psi_e\|_2\) and
\(\varepsilon_e=\|E_e\|_2<\sigma_e\). For nonzero vectors $w_1,w_2$, define
the projective angle, which ignores their arbitrary signs, by
\[
 \angle_{\rm proj}(w_1,w_2)
 =\arccos\frac{|w_1^\top w_2|}{\|w_1\|_2\|w_2\|_2}
 \in[0,\pi/2].
\]
A standard rank-one singular-subspace
perturbation bound \citep{stewart1990} gives, for the leading unit left
singular vector \(\widehat u_e\),
\begin{equation}
 \sin\angle_{\rm proj}(\widehat u_e,o_j)
 \leq \frac{\varepsilon_e}{\sigma_e-\varepsilon_e}.
 \label{app:v6eq:wedin}
\end{equation}
Let
\begin{equation}
 \alpha_j=\frac{\|L_{:,j}\|_2}{\|o_j\|_2},
 \qquad
 \mu_j=\min_{i\neq j}\angle_{\rm proj}(L_{:,j},L_{:,i}),
 \qquad
 d_e=\frac{\sqrt2\,\varepsilon_e}{\sigma_e-\varepsilon_e}.
 \label{app:v6eq:label-quantities}
\end{equation}
After aligning the arbitrary sign of \(\widehat u_e\), its Euclidean distance
from \(o_j/\|o_j\|_2\) is at most \(d_e\). Therefore, if
\begin{equation}
 d_e<\frac{\alpha_j\sin(\mu_j/2)}{1+\sin(\mu_j/2)},
 \label{app:v6eq:label-condition}
\end{equation}
the first block of \(\widehat u_e\) is nonzero and is strictly closer
projectively to \(L_{:,j}\) than to any other leadfield column.

\paragraph{Anatomical scale-calibration error.}
Write $b=\|o_j\|_2$, $u=o_j/b$, and align the sign of $\widehat u_e$
so that $\|\widehat u_e-u\|_2\leq d_e$.
Let $w$ and $\widehat w=w+\Delta w$ be their first sensor blocks,
so $\|w\|_2=\alpha_j$ and $\|\Delta w\|_2\leq d_e$.
At the correct label, $L_{:,j}=bw$ and the fitted scale obeys
\[
 \frac{\gamma_e}{b}-1
 =-\frac{(w+\Delta w)^\top\Delta w}{\|w+\Delta w\|_2^2}.
\]
Hence, for $d_e<\alpha_j$,
\begin{equation}
 \left|\frac{\gamma_e}{b}-1\right|
 \leq\frac{d_e}{\alpha_j-d_e},\qquad
 \frac{\|\widehat o_j-o_j\|_2}{\|o_j\|_2}
 \leq d_e+\frac{d_e}{\alpha_j-d_e}.
 \label{app:v6eq:scale-bound}
\end{equation}
The last inequality follows by writing
$\gamma_e\widehat u_e-bu=(\gamma_e-b)\widehat u_e+b(\widehat u_e-u)$.
The calibrated-column error is unchanged by flipping the singular-vector sign
and its fitted scale together.
The label condition in Eq.~\eqref{app:v6eq:label-condition} implies
$d_e<\alpha_j$, so it also ensures this calibration bound is finite.

\paragraph{Completion and temporal shift.}
Conditional on correct labels, let
\(\widehat Y=Y+E_Y\) and use the exact known \(\Omega\). Then
\begin{equation}
 \|\widehat{O}_T-O_T\|_2
 \leq \frac{\|E_Y\|_2}{\sigma_{\min}(\Omega)}.
 \label{app:v6eq:completion-bound}
\end{equation}
If
\(\widehat{O}_-=O_-+E_-\),
\(\widehat{O}_+=O_++E_+\),
\(\beta=\sigma_{\min}(O_-)\), and
\(\|E_-\|_2<\beta\), the shifted least-squares estimate satisfies
\begin{equation}
 \|\widehat F-F\|_2
 \leq
 \frac{\|E_+\|_2+\|E_-\|_2\|F\|_2}
 {\beta-\|E_-\|_2}.
 \label{app:v6eq:shift-bound}
\end{equation}

For completeness, \eqref{app:v6eq:label-condition} follows because restriction
to the first block cannot increase the vector perturbation. If $w$ is the
true first block and $\delta w$ its perturbation, then
$\sin\angle_{\rm proj}(w,w+\delta w)\leq
\|\delta w\|_2/(\|w\|_2-\|\delta w\|_2)$; the displayed condition makes
this smaller than \(\sin(\mu_j/2)\). Equation
\eqref{app:v6eq:completion-bound} follows from
\((\widehat Y-Y)\Omega^\dagger\). For \eqref{app:v6eq:shift-bound}, use
\[
 \widehat F-F
 =\widehat{O}_-^\dagger(E_+-E_-F),
 \qquad
 \|\widehat{O}_-^\dagger\|_2
 \leq(\beta-\|E_-\|_2)^{-1}.
\]

\paragraph{Proof of Proposition~\ref{prop:v6stability}.}
Set $\delta=\|E_Y\|_2/\sigma_{\min}(\Omega)$.
Equation~\eqref{app:v6eq:completion-bound} gives total observation-map error
at most $\delta$. If $E_O=\widehat O_T-O_T$ before enforcing the known
top block $L$, the resulting error after replacement is
$PE_O$, where $P=\diag(0_{m\times m},I_{m(T-1)})$ is an orthogonal
row projection. Thus its spectral norm does not increase.
Selecting the shifted row stacks is also contractive, so both
$\|E_-\|_2$ and $\|E_+\|_2$ are at most $\delta$.
If $\delta<\beta$, the estimated $O_-$ retains full column rank. Substituting
these inequalities into Eq.~\eqref{app:v6eq:shift-bound} proves
\[
 \|\widehat F-F\|_2\leq
 \frac{(1+\|F\|_2)\delta}{\beta-\delta}.
\]
The bounds separate four bottlenecks: contrast exposure \(\sigma_e\), scalp
visibility and projective label separation \((\alpha_j,\mu_j)\), coordinate
coverage \(\sigma_{\min}(\Omega)\), and dynamic inversion \(\beta\).
They are deterministic: finite-sample probabilities for target and rank
decisions require a response-noise model. In particular, an anatomy-only
score cannot control the unknown exposure factor \(v_e^\top R_s\).

\subsection{Scope of the identification claims}
\label{app:v6scope}

The results identify population-mean dynamics in a finite-dimensional linear
time-invariant model with known \(L,K\) and row-local, state-preserving changes
for one transition. The identified interactions are lagged source dynamics,
not an instantaneous causal graph or a particular innovation realization.
A rank-one contrast alone does not establish a mechanism change: a local
additive impulse can have the same source-specific left factor.

Timing and locality determine the contrast structure. Persistent changes
produce higher-order terms such as \(FD_e+D_eF+D_e^2\), rather than the single
outer product in \eqref{app:v6eq:contrast}. A diffuse rank-one change can also
share a local immediate scalp signature: if $h\in\ker L$, then
$(e_j+h)v^\top$ has immediate direction proportional to \(L_{:,j}\).
Within the stated protocol, exposure, anatomical calibration, coordinate
coverage, and temporal observability provide a constructive route to
identifying the source dynamics.

\section{Experimental protocols and complete results}
\label{app:v6experiments}

This appendix documents the full-reachability estimator comparison, its
general-row control, and a separate fresh-cohort/source-count study. Each study
is reported under its own cohort and protocol; results are not pooled across
studies. Responses and source activity are simulated using anatomical
observation operators from four Localize-MI geometries. The incomplete-
reachability experiment has its own dynamics generator and acquisition grid
in Appendix~\ref{app:v6partial}. Unless stated otherwise, the defaults below
concern the full-reachability comparison.

\subsection{Anatomy, dynamics, and replication}
\label{app:v6data}

The observation operators come from Localize-MI \citep{mikulan2020}, geometries
01, 03, 05, and 07. Each reduced operator $L\in\R^{8\times12}$ maps twelve
fixed cortical-normal candidate sources to eight independent sensor contrasts
formed from nine electrodes. Columns are normalized once before experiments,
fixing dimensionless source units. The biological input is anatomy, not
recorded scalp EEG or stereoelectroencephalography. Source number, locations,
orientations, and $L$ are known; source amplitudes are not calibrated as
physical currents.

To distinguish electrodes from independent contrasts, let
$L_{\rm elec}\in\R^{9\times12}$ denote the average-referenced electrode
operator in the same normalized source units. The experiment uses
$L=Q^\top L_{\rm elec}$, where $Q\in\R^{9\times8}$ satisfies
$Q^\top Q=I_8$ and $Q^\top\mathbf1_9=0$, and $\mathbf1_9$ has nine ones.
The nine-row electrode view and eight-row fitting operator therefore describe
the same centered sensor subspace without changing source units.

For the full-reachability comparison, $F\in\R^{12\times12}$ is generated by
independently masking uniform $[-0.6,0.6]$ entries with Bernoulli probability
$0.3$, replacing the diagonal with uniform $[0.05,0.45]$ entries, adding $0.3$
along a directed cycle, and rescaling its spectral radius to $0.85$. No system
is rejected on conditioning or recovery outcomes. The known initialization
map is $K=[e_1,e_2,e_3,e_4]$, using one-based source indices throughout the
manuscript (indices 0--3 in the implementation). The eight remaining
coordinates are the intervention targets; these labels are hidden from the
blind estimators. The primary noisy experiment uses
$D_e=-\eta_e e_{j_e}e_{j_e}^{\top}F$, with independent
$\eta_e\sim\operatorname{Unif}[0.2,0.6]$. This gain-suppression family is a
subclass of the general receiving-row model. A separate general-row control
is described in Appendix~\ref{app:v6arbitrary}.

Development used seed indices 0--3. Before inspecting these outcomes, the
internal protocol specified the four anatomy subsets and source normalization,
initialization and target directions, dynamics and gain-change generators,
$T=s=6$, the noise law and equal baseline/active allocation, and
$n\in\{16,64,256,1024\}$ with primary budget $n=256$.
It also specified the held-out reconstruction task, error metrics and joint
10\% success criterion. Development led to adding the equally informed
Spectral-local comparator, evaluated on the same development seeds before
confirmation. This addition did not change the Direct estimator, generators,
noise, budget grid or primary outcomes.

Main confirmation uses fresh seed indices
1000--1039, crossed with four fixed geometries, four budgets, and five methods:
640 anatomy--seed--budget cases and 3,200 method rows. A seed reuses dynamics,
intervention strengths, held-out states, and matched noise across geometries;
the 160 anatomy--seed pairs at one budget are not 160 independent systems or
patients. Dynamics, intervention, acquisition, and evaluation have separate
random-stream namespaces. Bootstrap comparisons resample the 40 complete
seed clusters, preserving anatomy and method pairing. The protocol
and implementation were fixed before confirmation, without external
preregistration.

\subsection{Acquisition and trial accounting}
\label{app:v6acquisition}

Each episode starts with controlled mean $\mathbb E[z_0]=K_{:,i}$ for input
direction $i\in\{1,\ldots,4\}$. An active episode uses $F+D_e$ at one insertion
transition $\tau_{\mathrm{int}}\in\{0,\ldots,5\}$ and $F$ at every other transition. Baseline
and active episodes all record thirteen samples, $\tau=0,\ldots,12$. For an
insertion at $\tau_{\mathrm{int}}$, relative output time $t=0,\ldots,T-1$ selects absolute
time $\tau=\tau_{\mathrm{int}}+1+t$. With six output lags and six insertion times, the
estimator's $48\times24$ baseline, shifted-baseline, and active blocks have
expectations $O_TR_s$, $O_TFR_s$, and $O_T(F+D_e)R_s$, where $T=s=6$.
Repeated baseline Hankel entries reuse the same recorded sample and are
not independent observations.

Figure~\ref{fig:v6state-change}B illustrates an intervention at
$\tau_{\mathrm{int}}=0$.

Single-trial process innovations have covariance $0.002^2I$. Sensor noise is
stationary Gaussian, with channel covariance $0.01^2\,0.3^{|i-j|}$ and temporal
autoregressive coefficient $0.4$. Episodes are independent within an
acquisition. Dividing every sensor and process innovation by the square root
of the condition's repetition count samples the exact Gaussian law of its
trial mean, including propagated process noise and temporal/channel
correlations. It is not independent matrix-entry noise. Across designs and
budgets, matched standard Gaussian draws are scaled by repetition count;
these are paired counterfactual acquisitions, not nested collections of trials.

For $n$ repeats per input, active mode, and insertion time, each of the four
baseline input conditions has $48n$ repeats. Thus
\[
 N_{\rm tot}=4(48n+8\cdot6n)=384n,
 \qquad n\in\{16,64,256,1024\}.
\]
Half the acquisition budget is baseline and half active. The resulting
budgets are 6,144, 24,576, 98,304, and 393,216 trial equivalents. Gaussian
averaging avoids materializing individual trials while preserving the
represented acquisition cost.

\subsection{Estimators and information access}
\label{app:v6comparators}

Every estimator receives the same measured arrays and known $L,K$ within a
case; additional information supplied to controls is specified below.

\paragraph{Direct contrast reconstruction.}
\label{app:v6algorithm}
Algorithm~\ref{alg:v6direct} in Section~\ref{sec:v6direct-estimation} gives the
procedure. This appendix specifies its numerical guards, repeated-target
handling, least-squares solves and information access.

Every column of $L$ must be nonzero. A leading contrast singular value at most
$\epsilon_\Delta=10^{-14}$ or first-block norm at most $\epsilon_u=10^{-12}$
produces an invalid estimate. These are numerical degeneracy guards, not
statistical detection thresholds. Each intervention supplies its own paired
column in $\Omega$ and $\widehat Y$, even if inferred labels repeat.
No deduplication, pooling, or one-to-one assignment is imposed. In the
correct-label analysis, the corresponding noiseless matrix is $Y=O_T\Omega$
with the same column ordering and multiplicities.

The implementation computes $\widehat O_T$ by solving
$\Omega^\top\widehat O_T^\top\simeq\widehat Y^\top$ and then solves the temporal
shift by ordinary least squares, rather than explicitly forming pseudoinverses.
Rank checks use \texttt{numpy.linalg.matrix\_rank} with its default tolerance,
and both solves use \texttt{scipy.linalg.lstsq} with its default singular-value
cutoff. Missing coordinate coverage or insufficient numerical observability
produces an invalid outcome.

\paragraph{Spectral-local reconstruction.}
This equally informed comparator first forms a shared Ho--Kalman realization
from the same noisy baseline Hankel matrix. It estimates each intervention
change in those common coordinates, extracts its leading left direction $h$,
and labels/scales it through $Ch=L_{:,j}$. It then solves the anatomical equations
$L\Phi=C$ and $\Phi[B,h_1,\ldots,h_8]=[K,e_{j_1},\ldots,e_{j_8}]$ by unweighted
least squares. The changes supplied to this procedure are estimated
realization-coordinate matrices, not the true $D_e$. The comparison isolates
direct contrast extraction versus extraction after full baseline realization;
it does not optimize over all estimators using the same information.

\paragraph{Information and misspecification controls.}
Target-oracle direct receives the true labels but otherwise uses the direct
estimator. The calibrated commutator receives the true full $D_e$, providing
strictly more intervention information; it is a particular unweighted spectral
estimator, not a statistically optimal oracle. The diagonal-only approximation
receives labels but replaces a receiving-row change with a diagonal trace
approximation. It tests the effect of misspecifying the intervention, rather
than serving as a competitive baseline. Performance of these controls therefore
does not rank information sets or establish that calibration is harmful.

Neither blind method receives true $F,D_e,\eta_e$, source states, or held-out
outcomes. No source-truth rotation, rescaling, or permutation is used to repair
the estimates. The evaluation concerns physical-coordinate recovery under
this shared information contract.

\subsection{Held-out reconstruction and confirmation results}
\label{app:v6recovery}

Figure~\ref{fig:v6recovery-overview}C--D presents the budget curves and paired source errors for this confirmation study.

For each system, 32 independent $z_{\rm init}\sim\mathcal N(0,I_{12})$ produce
eight held-out observations $y_\tau=LF^\tau z_{\rm init}+\xi_\tau$,
$\tau=0,\ldots,7$, with sensor-noise standard deviation $0.001$. There are no
new process innovations in these held-out mean trajectories. The estimator
infers $\widehat z_{\rm init}=\widehat O_8^{\dagger}y$ from the entire sequence
$y=[y_0^\top,\ldots,y_7^\top]^\top\in\R^{64}$ and propagates
$\widehat z_\tau=\widehat F^\tau\widehat z_{\rm init}$. True $z_{\rm init}$ is
not supplied. Let $Z_{\rm init}\in\R^{12\times32}$ collect the initial states
and $Z\in\R^{96\times32}$ stack eight true source states per episode column;
hats denote their reconstructed counterparts. We report
\[
 e_F=\frac{\Frob{\widehat F-F}}{\Frob F},\qquad
 e_z=\frac{\Frob{\widehat Z-Z}}{\Frob Z},\qquad
 e_{\rm init}=\frac{\Frob{\widehat Z_{\rm init}-Z_{\rm init}}}{\Frob{Z_{\rm init}}}.
\]
Here normalized root-mean-square error (NRMSE) equals the aggregate norm ratio.
Joint success requires $e_F\leq0.1$ and $e_z\leq0.1$. Invalid estimates retain
infinite raw errors and unsuccessful status in every denominator and quantile.
This is offline anatomical mean-trajectory reconstruction, not forecasting
or localization error in millimeters.

\begin{table}[ht]
\centering\small
\caption{Main confirmation at 98,304 trial equivalents. All rows contain 160
anatomy--seed cases. Errors are percentages; p90 uses the nearest observed
order statistic and includes invalid outcomes.}
\label{tab:v6fullcomparison}
\begin{tabular}{lrrrrr}
\toprule
Method & Valid & Success & Median $e_F$ & Median $e_z$ & p90 $e_z$\\
\midrule
Direct &160&155&2.925&2.203&4.279\\
Spectral-local &153&141&3.252&2.361&9.204\\
Target-oracle direct &160&155&2.925&2.203&4.279\\
Calibrated full-$D$ commutator &160&124&3.771&3.116&15.754\\
Diagonal-only approximation &160&0&250.988&47.935&54.743\\
\bottomrule
\end{tabular}
\end{table}

All 1,280 direct target labels are correct at the primary budget. Direct
source error is smaller in 123/160 paired cases, including the seven spectral-local
failures. Its success counts by geometry are 39/40, 40/40, 39/40, and 37/40.
The success difference between direct and spectral-local is 8.75 percentage points, with
a 10,000-resample seed-cluster percentile interval $[3.125,15]$
percentage points. This uncertainty is conditional on the simulation design,
not a biological-population interval. At the four increasing budgets direct
succeeds in 69, 126, 155, and 160 cases out of 160; spectral-local succeeds
in 39, 107, 141, and 156. Across all 3,200 rows, three direct and 61 spectral-local
estimates are invalid. The first tested budget meeting 90\% success on every
geometry differs by a factor of four; the grid does not estimate a continuous
sample-complexity ratio. At the primary development budget, the corresponding
counts were 12/16 for direct and 13/16 for spectral-local; development cases
are kept separate from confirmation.

\paragraph{Sensitivity to model assumptions.}
In separately fixed stress tests with 20 paired seeds per geometry, a 1\%
fitting-leadfield perturbation yields 79/80 direct successes at 98,304 trial equivalents
and 80/80 at 393,216. Adding a neighboring receiving-row change at 0.1 of the
primary suppression coefficient yields 65/80 and 69/80, respectively:
additional repeats do not remove the bias from nonlocal changes. The 0.1
coefficient ratio is not a measured response-energy ratio. The separate noisy
test in Section~\ref{sec:v6partial} and Appendix~\ref{app:v6partial} examines
incomplete baseline reachability using general row changes, since proportional
gain suppression cannot expose a strictly unreachable receiving row.

\subsection{General-row control}
\label{app:v6arbitrary}

A separate study, specified before outcomes, uses seed indices 5000--5039 per geometry, with
geometry included in the dynamics/noise namespace. Its eight-target
arbitrary-row arm is a \emph{prespecified secondary control}: the broader
study compared target counts under an additional exact-gain assumption.
This arm tests the general-row model as a secondary control, separately
from the main cohort. It contains 160
independent simulated systems across four fixed geometries, each measured
at 98,304 trial equivalents.

Each incoming-row direction is independently isotropic and norm-matched to
$\eta_j\|F_{j,:}\|_2$, rather than proportional to $F_{j,:}$. Direct recovers
157/160 systems, with source median/p90 2.106\%/4.157\%; spectral-local
recovers 139/160, with 2.915\%/9.989\%. Valid counts are 160/160 and 153/160,
respectively. These results support general receiving-row changes without
requiring the estimator to know the row coefficients or assume proportional
suppression.

\subsection{Fresh-seed check and source-count scaling}
\label{app:v6fresh}
\label{app:v6scaling}
An exploratory study, with its protocol fixed before outcomes,
tests the same direct estimator on fresh systems and larger dictionaries.
The configurations $(q,m,r,E)=(12,8,4,8),(24,16,8,16),(48,32,16,32)$ use
9, 17, and 33 electrodes, respectively. Thus sources, measurements,
initialization directions, and intervention targets co-scale; sensor count
is not held fixed. For each anatomy, nested electrode/source subsets are
chosen by deterministic farthest-point selection starting at maximum height,
without using activity or recovery outcomes. After removing the reference
component and normalizing columns, the smallest operator reproduces the
main $12$-source/$8$-measurement operator within $10^{-14}$.

Twenty fresh seed blocks 6000--6019 are shared across four anatomies and
budget regimes, giving 80 cases per configuration. Different sizes use
different dynamics, not nested subgraphs. The dynamics recipe, unknown
gain-suppression interventions, noise law, horizons $T=s=6$, and held-out
evaluation follow the full-reachability experiment. Both blind methods
receive identical arrays. Spectral-local uses an algebraically equivalent
singular-value decomposition for its final anatomical least-squares solve,
avoiding a large Kronecker matrix. Equivalence to the original implementation
was checked before the comparison.

Each active input/target/insertion condition has $n$ repeats; each baseline
input has $Esn$ repeats. Hence $N_{\rm tot}=2rEsn$. A fixed total of 98,304
trial equivalents gives $n=256,64,16$ at the three sizes. Holding $n=256$
instead raises total cost by factors of four and sixteen. These regimes
share systems and scaled noise streams. The $q=12$ cell belongs to both
regimes and is generated and counted only once.

\begin{table}[htbp]
\centering\small
\setlength{\tabcolsep}{3.5pt}
\caption{Full fresh-cohort grid. Each entry compares direct/spectral-local
on 80 cases. $N_{\rm tot}$ counts trial equivalents. Errors are percentages;
medians and p90 are inverse empirical
cumulative-distribution quantiles including infinite invalid outcomes.
The shared $q=12$ row appears once.}
\label{tab:v6scaling}
\begin{tabular}{@{}lrrcccc@{}}
\toprule
$q/m$ & $n$ & $N_{\rm tot}$ & Success & Valid & Median $e_z$ & p90 $e_z$\\
\midrule
12/8 &256&98,304&63/57&73/69&2.69/3.48&32.85/$\infty$\\
24/16&64&98,304&42/14&80/73&7.35/13.08&12.60/34.16\\
48/32&16&98,304&0/0&77/0&26.02/$\infty$&38.02/$\infty$\\
\midrule
24/16&256&393,216&77/65&80/80&3.62/5.33&6.42/10.77\\
48/32&256&1,572,864&60/19&80/75&6.32/8.94&10.80/39.60\\
\bottomrule
\end{tabular}
\end{table}

The fresh $q=12$ cohort gives 63/80 direct and 57/80 spectral-local successes,
compared with 155/160 and 141/160 in the main confirmation at the same
size and budget. Cohorts are reported separately. Post-hoc diagnostics find
three seed blocks with no direct successes across their four anatomies and
small minimum exposure, approximately $0.00128$, $0.00353$, and $0.00366$
versus cohort median $0.06691$. This association is consistent with exposure-
dependent conditioning; it does not isolate exposure as the sole cause.

At fixed total cost, neither method succeeds at 48 sources. Increasing
repeats yields 77/80 direct successes at 24 sources and 60/80 at 48, compared
with 65/80 and 19/80 for spectral-local. At 48 sources and the larger budget,
all direct target labels are correct, but only 60/80 systems meet the joint
error criterion. Correct labels therefore do not alone ensure accurate
dynamics. All 48 noiseless method checks pass on the first two prespecified
seeds per size and anatomy, with maximum relative $F$ error below
$1.97\times10^{-13}$. The median true observability singular value
$\beta=\sigma_{\min}(O_-)$ falls from 0.270 to 0.185 to 0.109. These
truth-based diagnostics are not supplied to either fitter.

Exploratory 10,000-resample intervals preserve the 20 complete seed
clusters and paired methods. Success differences (direct minus spectral-local) are
7.5 percentage points ($[0,17.5]$) at $q=12$, with the interval including zero;
at $q=24$ they are 35
($[18.75,51.25]$) at fixed budget and 15 ($[5,27.5]$) at fixed repeats; at
$q=48$ and fixed repeats they are 51.25 ($[36.25,65]$). Both methods have
zero successes in the fixed-budget $q=48$ cell. The study demonstrates
reconstruction beyond twelve sources while exposing the cost and conditioning
of this co-scaling protocol; it does not identify a minimum required budget
or isolate source count from the accompanying changes in sensors and inputs.

\subsection{Computational provenance}
\label{app:v6verification}

All 640 measured datasets and 3,200 fitted outcomes in the main study were
independently checked, including an independent NumPy direct
completion and separate metric calculation. The general-row study has its
own numerical verification. Figure-data validation checks complete
metric-table grids, source checksums, validity/success semantics, and
summary algebra for Figure~\ref{fig:v6recovery-overview}C--D.
The scaling study is verified separately. Infinite spectral-local failures in
panel D are marked in a
separate, nonmetric upper strip, not assigned finite measured errors.

The scaling study's validation covers 800 metric rows, 400 saved model
archives, source/geometry hashes, and equivalence of the two implementations
of the spectral solve. Its 12 summary rows were checked against all 800 metric rows,
including success, validity, and unconditional error quantiles. The shared
$q=12$ rows are deduplicated only in presentation, not treated as a new
replication. Appendix~\ref{app:v6partial} documents the incomplete-reachability
study and its independent numerical verification.

\section{Recovery under incomplete baseline reachability}
\label{app:v6partial}

This separate simulation tests the distinction between baseline
reachability and exposed coordinate coverage in Theorem~\ref{thm:v6identification}.
It uses the same four anatomy-derived operators as Appendix~\ref{app:v6data},
but a different dynamics generator, general receiving-row changes, and a new
noise/budget grid. All responses and dynamics are simulated. The protocol
was fixed before outcomes.

\subsection{Exact baseline reachability restriction}

There are twelve sources, four initialization directions, eight independent
sensor contrasts, and horizons $T=s=6$. Draw $G\in\R^{12\times12}$ with
independent standard Gaussian entries. Starting from
$0.5I_{12}+0.25G/\sqrt{12}$, set the lower-left $8\times4$ block to zero and
rescale the spectral radius to $0.85$. This produces the block form in
Eq.~\eqref{eq:v6partial-block}. The known $K=[I_4;0]$ has rank four, and
$F^\tau K=[A^\tau;0]$ for every nonnegative $\tau$. Baseline reachability
therefore has dimension exactly four even at infinite horizon. Process noise
can excite other states in individual trials without changing this
\emph{conditional-mean} restriction. The estimator is not told the zero-block
pattern.

For each of eight intervention modes, draw independent unit vectors
$a_e\in\R^4$ and $b_e\in\R^8$. Assign the modes a random permutation of
targets $5,\ldots,12$, hidden from fitting. The exposed and unexposed
row-change vectors are respectively
\begin{equation}
 v_e^{\rm exp}=\begin{bmatrix}0.35a_e\\0.20b_e\end{bmatrix},
 \qquad
 v_e^{\rm unexp}=\begin{bmatrix}0\\\sqrt{0.35^2+0.20^2}\,b_e\end{bmatrix}.
 \label{eq:v6partial-rows}
\end{equation}
Both produce nonzero local changes $D_e=e_{j_e}v_e^\top$ of the same
Frobenius norm, approximately $0.4031$. The first satisfies
$(v_e^{\rm exp})^\top K=0.35a_e^\top\ne0$; the second satisfies
$(v_e^{\rm unexp})^\top R_s=0$ for every horizon. The design uses general row
changes, not proportional suppression of an unreachable row. The three
fixed conditions are all eight exposed, only the final mode unexposed, and
all eight unexposed. They share $F,L,K$, targets, and row directions, with
no additive offset or further drive.

Eight fresh seed indices, 6100--6107, are used per anatomy. Anatomy is part
of the dynamics, intervention, acquisition, and evaluation random-stream
namespaces: the 32 cases are independent simulated draws conditional on four
fixed geometries, not independent participants. Each system is drawn once,
without conditioning-based selection or outcome-based replacement. Exact
audits of all 96 system/condition combinations give
\[
 \rank L=8,\quad \rank R_6=\rank H_0=4,\quad \rank O_-=12.
\]
Observability is checked, not implied by the block construction. Its smallest
singular value ranges from $0.0755$ to $0.2508$. Leadfield columns are nonzero
and pairwise nonproportional; maximum pairwise absolute line cosines range
from $0.8209$ to $0.9904$ across geometries. In the all-exposed condition,
$\Omega=[K,E_J]$ is a column permutation of $I_{12}$. All 32 noise-free blind fits
recover the targets, with maximum relative dynamics error
$3.36\times10^{-15}$.

\subsection{Acquisition and evaluation}

The trial law and shared response lags follow Appendix~\ref{app:v6acquisition}.
The two single-trial sensor/process standard-deviation pairs are
$(0.01,0.002)$ and $(0.04,0.008)$ in normalized units. Sensor covariance
retains channel correlation $0.3^{|i-j|}$ and temporal autoregressive
coefficient $0.4$; process innovations are independent isotropic Gaussian.
We draw exact distributions of Gaussian trial means, including covariance
propagated through each intervention, rather than adding independent noise
to response-matrix entries.

Each acquisition has twelve independent equal batches, with $n/12$ repeats
per input/mode/insertion condition in each batch; the estimators use their
overall mean. For $n=12,48,192$, half the total cost is assigned to shared
baseline episodes:
\[
 N_{\rm tot}=2\cdot4\cdot8\cdot6\,n=384n
 \in\{4{,}608,18{,}432,73{,}728\}.
\]
For example, $n=48$ repeats per active initialization/target/insertion
condition give $N_{\rm tot}=384\cdot48=18{,}432$ trial equivalents, including
the matched baseline allocation.
The same standard Gaussian streams are used across budgets, noise levels,
and exposure conditions. These are paired counterfactual acquisitions, not
nested recordings or independent replications. In particular, the higher-
noise $n=192$ acquisition has the same realized mean errors as lower-noise
$n=12$, because $4/\sqrt{192}=1/\sqrt{12}$. Those repeated cells are not
pooled as independent evidence.

The direct and spectral-local estimators defined in Appendix~\ref{app:v6comparators}
receive only $L,K,q$
and measured response blocks. Neither receives true targets, row coefficients,
states, reachability rank, or a true-parameter initialization. Direct matching
is independent across modes without a supplied one-to-one assignment.

An additional gated-direct variant accepts the same direct estimate only if
all eight modes pass an observed-data exposure test. Write the twelve batch
contrasts as $\widehat{\Delta H}_e^{(b)}$, $b=1,\ldots,12$, and define
\[
 T_e(\omega)=\left\|\frac1{12}\sum_{b=1}^{12}
 \omega_b\widehat{\Delta H}_e^{(b)}\right\|_{\rm F}^2,
 \qquad \omega_b\in\{-1,+1\}.
\]
The raw $p$-value is the fraction of sign patterns with statistic at least
$T_e(\mathbf1_{12})$, counting ties conservatively with a floating-point
tolerance. Exact enumeration of $2^{11}$ patterns suffices because a global
sign leaves $T_e$ unchanged. Holm adjustment at level $0.05$ is applied to
the eight nulls $\Delta H_e=0$; all eight must be rejected to pass.
Under an unexposed mode, independent centered-Gaussian batch contrasts are
jointly invariant to batchwise sign changes. Whole matrices, not their
correlated entries, are the sign units; a zero mean alone would not ensure
this symmetry. Baselines are independent across batches and shared across
modes within a batch, a dependence allowed by Holm adjustment. Batching
adds no trials. The gate tests observable exposure, not correct labels,
rank, anatomical calibration, or biological locality.

Evaluation uses 32 independent standard-Gaussian initial states covering all
twelve coordinates. Their length-eight mean trajectories are observed through
$L$ with independent sensor noise of standard deviation $0.001$. These test
states and observations are separate from fitting and shared across conditions.
As in Section~\ref{sec:v6experiments}, each fitted model estimates the initial
state from all eight observations. The endpoint is offline source reconstruction.

\subsection{Recovery under full exposure}

Table~\ref{tab:v6partial-errors} contains every all-exposed case. Direct and
gated-direct results coincide: all 32 estimates pass the gate, are finite,
and have correct targets at each cell. Success requires
both errors in Eq.~\eqref{eq:v6metrics} to be at most $10\%$. Reported 90th
percentiles use the nearest-order-statistic convention (index
$\operatorname{round}[0.9(32-1)]$ in a zero-based sorted list), without
interpolation or success-only filtering.

\begin{table}[ht]
 \centering
 \small
 \setlength{\tabcolsep}{4pt}
 \caption{All-exposed direct contrast reconstruction. Error columns are percentages;
 $N_{\rm tot}$ counts trial equivalents.
 Each row contains 32 cases; noise/budget rows are paired. SD denotes standard
 deviation; med. and p90 denote median and the specified 90th percentile.}
 \label{tab:v6partial-errors}
 \begin{tabular}{@{}lrrcrrrr@{}}
 \toprule
 Sensor/process SD & $n$ & $N_{\rm tot}$ & Success &
 \multicolumn{2}{c}{$e_F$ (\%)} & \multicolumn{2}{c}{$e_z$ (\%)}\\
 \cmidrule(lr){5-6}\cmidrule(l){7-8}
 & & & & med. & p90 & med. & p90\\
 \midrule
 0.01/0.002 & 12  & 4,608  & 24/32 & 6.33 & 10.29 & 8.08 & 11.85\\
 0.01/0.002 & 48  & 18,432 & 32/32 & 3.24 & 5.24  & 3.97 & 6.02\\
 0.01/0.002 & 192 & 73,728 & 32/32 & 1.66 & 2.64  & 1.99 & 3.04\\
 0.04/0.008 & 12  & 4,608  & 0/32  & 26.50 & 37.15 & 29.52 & 40.88\\
 0.04/0.008 & 48  & 18,432 & 3/32  & 12.79 & 19.70 & 15.34 & 22.50\\
 0.04/0.008 & 192 & 73,728 & 24/32 & 6.33 & 10.29 & 8.08 & 11.85\\
 \bottomrule
 \end{tabular}
\end{table}

Repetition improves reconstruction without changing baseline reachability.
Fourfold larger noise produces a substantial failure regime despite complete
coverage, correct labels, and unchanged exact ranks. The final higher-noise
row repeats the first lower-noise row by design. This distinguishes exact
identification from the precision needed for successful finite-budget recovery.

The spectral-local method returns zero valid or successful
reconstructions in every all-exposed cell. Its full-order baseline
factorization attempts to infer twelve coordinates from a population Hankel
matrix of rank four; incorrect or repeated target labels fail coverage. It
therefore serves here as a boundary comparator outside its full-baseline-rank
assumption, not as a test of every possible joint or latent-system estimator.

\subsection{Exposure controls and false confidence}

Table~\ref{tab:v6partial-controls} retains both unexposed controls and applies
at every paired noise/budget cell. Let $J_{\rm exp}$ denote targets with
nonzero population contrasts. Then $[K,E_{J_{\rm exp}}]$ has ranks 12, 11,
and 4 in the three conditions. Nonzero intervention strength does not itself
supply the missing response column.

\begin{table}[ht]
 \centering
 \small
 \setlength{\tabcolsep}{4pt}
 \caption{Exposure controls, with 32 cases per paired cell. ``Gate'' means
 all eight modes pass; ``Finite'' counts valid numerical estimates, not
 successful recovery. Success counts are the same for direct and gated-direct.
 The three false passes are the same systems across grid cells.}
 \label{tab:v6partial-controls}
 \begin{tabular}{@{}lccccc@{}}
 \toprule
 Condition & Coverage rank & Gate & Finite direct & Finite gated & Success\\
 \midrule
 All exposed & 12 & 32/32 & 32/32 & 32/32 & Table~\ref{tab:v6partial-errors}\\
 One unexposed & 11 & 3/32 & 1/32 & 1/32 & 0/32\\
 All unexposed & 4 & 0/32 & 0/32 & 0/32 & 0/32\\
 \bottomrule
 \end{tabular}
\end{table}

In the one-unexposed condition, the gate falsely accepts the missing mode
in 3/32 systems (9.375\%). This observed fraction does not establish empirical
error at or below 5\%, nor does it contradict the conditional level-5\% test.
In the all-unexposed condition, one system has one falsely rejected mode,
but none passes all eight. These events recur on the paired noise streams;
they are not independent confirmations at each budget or noise level.

One one-unexposed system (geometry 07, seed 6107) passes the gate and gives
a finite direct and gated estimate because a noise direction happens to
match the missing target.
At the largest lower-noise budget, its dynamics and source errors are
74.79\% and 54.65\%. Correct target labels and a passed exposure gate
therefore do not establish calibration of the temporal response. Neither
unexposed control yields a successful recovery: increasing repeats cannot
make its zero population contrast nonzero.

To distinguish information loss from estimator failure, an independent
validator constructs a different $F_S=SFS^{-1}$ preserving $L,K$ and all
included \emph{mean response blocks} for each unexposed control system.
Exposed changes remain receiving-row local; unexposed changes retain their
original nonzero rows and remain invisible on the baseline reachable space.
Across 64 witnesses, maximum mean-response disagreement is
$1.12\times10^{-16}$ and minimum relative dynamics difference is 2.13\%.
These witnesses concern the specified mean-response protocol, not equality
of full noisy distributions or arbitrary repeated-switch schedules. They
are separate constructions, not applications of the coordinate-gauge
proposition that drop unexposed nonzero modes from its preservation conditions.

\subsection{Independent numerical audit}

The study contains 576 acquisitions and 1,728 fitted records for direct,
gated-direct, and spectral-local; all outcomes are retained. An independent
validator recomputed all fitted metrics, all 4,608 mode $p$-values
using the full $2^{12}$ sign orbit, 72 acquisitions, 96 exact-system audits,
and 64 ambiguity witnesses. Every $p$-value and Holm decision matched,
and every check passed. Maximum metric disagreements were
$9.72\times10^{-17}$ for dynamics and
$6.11\times10^{-16}$ for source error; source, geometry, and output hashes
are preserved. Together, the exposed and unexposed conditions test the information
requirement of the theorem under the same anatomy and acquisition law.

\section{Injective leadfields and dynamical excitation}
\label{app:v6injective}

An injective leadfield separates instantaneous state readout from
identification of the state-transition rule. We derive the consequences for
identification and stability, then test excitation and sensor count on the
same four anatomical geometries using new simulated systems and matched
physical-channel noise.

\subsection{Identification with full-column-rank leadfields}
\label{app:v6injective-theory}

Assume the known $L\in\R^{m\times q}$ has full column rank $q$; thus $m\ge q$.
When $m=q$, this assumption requires invertibility, not just equal dimensions.

\begin{corollary}[Identification with an injective leadfield]
\label{cor:v6injective}
If $\rank L=q$, its columns are nonzero and pairwise nonproportional, and
$\rank O_-=q$ for every $T\ge2$. Hence exposed changes with $\rank\Omega=q$
satisfy Theorem~\ref{thm:v6identification} without an additional
observability condition. For square $L$ and $T=2$, reconstruction reduces to
$F=L^{-1}O_{\rm bottom}$, where $O_{\rm bottom}=LF$ is the second sensor block
of $O_2$.
\end{corollary}

Every column is nonzero and no two columns are proportional, because either
defect would create a nonzero vector in $\ker L$. For $T\ge2$, the stack
$O_-$ contains $L$, and therefore
\begin{equation}
 \|O_-a\|_2^2\ge\|La\|_2^2
 \ge\sigma_{\min}(L)^2\|a\|_2^2\qquad(a\in\R^q).
 \label{eq:v6injective-observability}
\end{equation}
Consequently $\rank O_-=q$ and
$\sigma_{\min}(O_-)\ge\sigma_{\min}(L)>0$ without a dynamical observability
condition. The contrast factorization in Eq.~\eqref{eq:v6contrast} is
unchanged: it follows from receiving-row locality, not from noninjectivity.
Theorem~\ref{thm:v6identification} thus applies with exposure and full
coordinate coverage, and with its anatomical separation and temporal-rank
requirements automatically satisfied. These are exact-rank statements;
near-collinear columns can still make noisy labeling difficult.

For $T=2$, the observation stack is simply $O_2=[L;LF]$. Let
$\widehat O_{\rm bottom}\in\R^{m\times q}$ denote the lower block of the
estimated stack after replacing its first block by the known $L$. The shift
solve is then $\widehat F=L^\dagger\widehat O_{\rm bottom}$, or
$L^{-1}\widehat O_{\rm bottom}$ when $m=q$. A sharper conditional bound
follows because this regressor is exact.

\begin{proposition}[Conditional stability with an injective known readout]
\label{prop:v6injective-stability}
Assume correct target labels, $\rank\Omega=q$ and $\rank L=q$. Construct
$\widehat O_T=\widehat Y\Omega^\dagger$, overwrite its first block by $L$, and
solve the temporal shift by full-rank least squares. With $\delta$ defined in
Eq.~\eqref{eq:v6errorlevel}, for every $T\ge2$,
\begin{equation}
 \|\widehat F-F\|_2\le
 \frac{(1+\|F\|_2)\delta}{\sigma_{\min}(L)}.
 \label{eq:v6injective-stability}
\end{equation}
For $T=2$, the sharper bound is
\begin{equation}
 \|\widehat F-F\|_2\le
 \frac{\|\widehat Y-Y\|_2}
 {\sigma_{\min}(L)\sigma_{\min}(\Omega)}.
 \label{eq:v6injective-two-block}
\end{equation}
Neither bound needs $\delta<\sigma_{\min}(O_-)$.
\end{proposition}
\begin{proof}
Before overwriting, the stack error is
$(\widehat Y-Y)\Omega^\dagger$, of norm at most $\delta$. Overwriting multiplies
this error on the left by the orthogonal projection that zeros the first
block, so it cannot increase its norm. Write $E_-$ and $E_+$ for the two
shifted error stacks; both have norm at most $\delta$. Since
$\widehat O_-$ contains the exact $L$, the argument in
Eq.~\eqref{eq:v6injective-observability} gives
$\sigma_{\min}(\widehat O_-)\ge\sigma_{\min}(L)$. The least-squares identity
\[
 \widehat F-F=\widehat O_-^\dagger(E_+-E_-F)
\]
then proves Eq.~\eqref{eq:v6injective-stability}. For $T=2$, $E_-=0$ and
$\widehat O_-=L$, proving Eq.~\eqref{eq:v6injective-two-block}.
\end{proof}

With exact $L$, the shift regressor stays full rank regardless of the
column-estimation error. Correct labeling and calibration remain required,
and a small $\sigma_{\min}(L)$ still amplifies errors. The numerical study
below uses $T=s=6$ and Algorithm~\ref{alg:v6direct}; the special $T=2$ formula is
an analytical consequence.

\subsection{Coordinate ambiguity and dynamical excitation}
\label{app:v6injective-baseline}

Full column rank makes $L(\Delta S)=0$ imply $\Delta S=0$. Thus the displayed coordinate
family in Proposition~\ref{prop:v6gauge} has $g_{\rm row}=0$ even when $\Omega$
is rank deficient. This does not imply that all dynamics matching the
available responses are identical. To characterize what baseline means
alone leave undetermined, define the baseline reachable subspace
\begin{equation}
 \mathcal R=\operatorname{span}\{K,FK,F^2K,\ldots\},
 \qquad d=\dim\mathcal R.
 \label{eq:v6baseline-subspace}
\end{equation}
Here the span is of all columns of the displayed matrices.

\begin{proposition}[Exact ambiguity of baseline mean dynamics]
\label{prop:v6baseline-affine}
With known full-column-rank $L$ and known $K$, a matrix $F'$ has the same
baseline mean responses $LF'^\tau K=LF^\tau K$ for every $\tau\ge0$ if and only if
\begin{equation}
 F'=F+\Delta F,\qquad \Delta F\,\mathcal R=\{0\}.
 \label{eq:v6baseline-affine}
\end{equation}
Without additional restrictions on $F'$, this affine family has dimension
$q(q-d)$.
\end{proposition}
\begin{proof}
Left inversion of $L$ makes equality of baseline outputs equivalent to
$F'^\tau K=F^\tau K$. If these states agree, then for every $\tau$,
$(F'-F)F^\tau K=F'F'^\tau K-FF^\tau K=0$, proving necessity. Conversely, $K$ has
columns in $\mathcal R$, this subspace is invariant under $F$, and
$\Delta F$ vanishes on it. Induction gives $F'^\tau K=F^\tau K$ for all $\tau$.
Finally, $\Delta F$ is an arbitrary linear map from
$\R^q/\mathcal R$ to $\R^q$, giving $q(q-d)$ free dimensions.
\end{proof}

For $q=12,d=4$, there are 96 such directions before additional model
constraints. If $F$ is strictly stable, sufficiently small perturbations in
this family remain stable. The statement concerns the controlled conditional
means used in this paper. Full stochastic recordings and noise covariances
can carry additional information under additional assumptions; those are not
identified or used here. If instead $d=q$, baseline state transitions with
a spanning set of predecessors determine $F$ without any interventions.

\subsection{Identification through response propagation}
\label{app:v6injective-propagation}

Injectivity also permits an alternative construction that does not assemble
$O_T$ column by column. Left-invert each measured contrast block and define
\begin{equation}
 C_{e,t,\tau_{\mathrm{int}}}:=L^\dagger\Delta H_e[t,\tau_{\mathrm{int}}]
 =F^tD_eF^{\tau_{\mathrm{int}}} K\in\R^{q\times r},\qquad t=0,\ldots,T-1.
 \label{eq:v6injective-contrast-state}
\end{equation}
All consecutive post-insertion blocks obey
$C_{e,t+1,\tau_{\mathrm{int}}}=FC_{e,t,\tau_{\mathrm{int}}}$. The unknown modified transition has already
ended at relative output time $t=0$. Stack baseline predecessor states and selected
$C_{e,t,\tau_{\mathrm{int}}}$ for $0\le t<T-1$ into a matrix
$W_-$ with $q$ rows; stack their corresponding successors into an
equally sized matrix $W_+$. Then
\begin{equation}
 W_+=FW_-,\qquad
 F=W_+W_-^\dagger\quad\text{if }\rank W_-=q.
 \label{eq:v6injective-propagation}
\end{equation}
This state-readout regression uses only transitions under baseline $F$ and
requires neither target labels nor changed coefficients. Its left-inversion
step is specific to injective $L$: when $m<q$, $L^\dagger L\ne I_q$, and
Eq.~\eqref{eq:v6injective-contrast-state} need not recover the true source
states.

The conditions $\rank\Omega=q$ and $\rank W_-=q$ are distinct. The first is
sufficient for the paper's direct column-completion construction, but is
not necessary for the second route: propagation may turn a few exposed
directions into many linearly independent states. For example, take
\begin{equation}
 L=I_3,\quad K=e_1,\quad
 F=\begin{bmatrix}0.6&0.2&0.1\\0&0.4&0.3\\0&0.25&0.5\end{bmatrix},
 \quad D=e_2e_1^\top.
 \label{eq:v6injective-three-source}
\end{equation}
Baseline means reach only $e_1$. One intervention gives $DK=e_2$; after one
baseline transition it gives $Fe_2=0.2e_1+0.4e_2+0.25e_3$. The three
predecessors $e_1,e_2,Fe_2$ span $\R^3$, and their measured successors
determine $F$ for $T\ge3$, although $\rank[K,e_2]=2$. This example is also
an exact unit test: the direct-coverage target count is sufficient for
column completion but is not a universal intervention lower bound.

\subsection{Matched sensor-count experiment}
\label{app:v6injective-experiment}

\paragraph{Geometry and source units.}
The four Localize-MI geometries correspond to participants 01, 03,
05 and 07 \citep{mikulan2020}. We retain each geometry's twelve fixed
cortical-normal sources and its geometry-only electrode ordering.
Nested sets of 9, 13 and 17 physical electrodes give $m=8,12,16$
independent average-reference contrasts. Subtracting a montage's average
makes the sum of its channels zero, so the square case requires
\emph{13 physical electrodes for 12 independent measurements}.
A geometry-only 12-electrode control has $m=11$. All four geometries have
the expected ranks $8,12,12$ for the three fitted sensor arms and rank 11
for that control. Increasing sensor count does not change source positions,
orientations or units: all column normalizations are inherited from the
original nine-electrode model, rather than recomputed for each montage.

\paragraph{Systems and input regimes.}
There are eight fresh simulation seeds, 7200--7207, per anatomy, giving 32
anatomy--seed cases per cell, not 32 participants. The stable block dynamics
are generated as in Appendix~\ref{app:v6partial}, with spectral radius
$0.85$. Each system is tested with the same $F$ but two known initialization maps:
\begin{equation}
 K_{\rm partial}=\begin{bmatrix}I_4\\0\end{bmatrix},\qquad
 K_{\rm full}=\frac1{\sqrt2}\begin{bmatrix}I_4\\K_{\rm aux}\end{bmatrix},
 \qquad K_{\rm aux}\in\R^{8\times4},\quad K_{\rm aux}^\top K_{\rm aux}=I_4.
 \label{eq:v6injective-inputs}
\end{equation}
The seeded auxiliary input block $K_{\rm aux}$ is fixed before evaluation.
It is distinct from the electrode-reference basis $Q$ in Appendix~\ref{app:v6experiments}.
Both maps have four orthonormal
columns. Partial-input reachability is exactly four at every horizon; full
inputs give $\rank R_6=12$ in all cases, without replacing failed draws.
The eight receiving-row changes, hidden from every estimator, are identical
across input and sensor arms: targets are a seeded permutation of $5$--$12$,
with coefficients distributed as $v_e^{\rm exp}$ in
Eq.~\eqref{eq:v6partial-rows}. These are general row changes, not an assumed
gain law. Input regimes also differ in conditioning: the ratio of largest
to smallest singular value of $\Omega$ is 1 versus 2.414. Under full inputs,
$\sigma_{\min}(R_6)$ ranges from
0.00202 to 0.02614 (median 0.00818), and its condition number ranges from
70.52 to 779.83 (median 211.74). The two input designs therefore differ in both reachability and
conditioning.

\paragraph{Acquisition and equal-budget controls.}
Use $T=s=6$, $n\in\{12,48,192\}$ repetitions per input, intervention and
insertion time, and $48n$ shared baseline repetitions per input. Total cost
is $4(48n+8\cdot6n)=384n$, or 4,608, 18,432 and 73,728 trial equivalents,
each representing an episode of 13 time points. A separately acquired
no-intervention control
allocates the entire same budget to baseline: $96n$ repetitions per input.
It is not the shared baseline half with its cost relabeled.

The Gaussian trial-mean sampler retains process propagation,
shared baseline lags and within-episode correlations. Single-trial
sensor/process standard deviations are $(0.01,0.002)$ and $(0.04,0.008)$.
Noise is generated once in 17 physical channels, then subset and projected
to each montage's orthonormal average-reference basis. Sensor covariance
is $\sigma_\nu^2 0.3^{|i-j|}$ in the fixed electrode-selection order, with
temporal autoregressive coefficient $0.4$; this is an assumed covariance,
not measured physiology. Sensor samples and process noise are genuinely
shared across sensor arms. Common random numbers also pair the input,
noise and repetition arms. These are not independent replications. Equal
episode counts and durations do not equalize hardware cost or the number
of recorded scalar measurements.

\paragraph{Estimators and information access.}
Direct contrast reconstruction uses Algorithm~\ref{alg:v6direct}. Six controls first solve the
known-$L$ source readout by least squares, then estimate $F$: baseline-only,
baseline plus post-insertion contrasts, and separately acquired all-budget
baseline, each using ordinary least squares (OLS) or ridge selected by
generalized cross-validation (GCV). The interventional pair uses exactly the
same response blocks, $L$ and $K$ as direct reconstruction, but no target
labels or row coefficients. It implements Eq.~\eqref{eq:v6injective-propagation}.
The shared-baseline pair uses only half the acquisition budget and is an
information-restriction diagnostic; the all-baseline pair tests an equal
total budget with a different allocation. Baseline response lags are
deduplicated and the known initial state is replaced by $K$; later states
are not supplied by an oracle.

For noisy predecessor/successor estimates $\widehat W_-,\widehat W_+$,
ridge fits minimize
$\|\widehat W_+-F\widehat W_-\|_{\rm F}^2+\lambda\|F\|_{\rm F}^2$.
GCV selects among $\lambda=\alpha\sigma_{\max}(\widehat W_-)^2$,
$\alpha\in\{0,10^{-8},10^{-7},\ldots,10^1\}$, by minimizing the training
residual divided by the squared residual degrees of freedom. No truth or
held-out outcomes select $\alpha$. Only the dynamics regression is
regularized; the $L$ readout is unregularized least squares with a numerical
singular-value cutoff. Correlated columns and errors in the estimated
predecessors make GCV a heuristic sensitivity check, not an optimal
errors-in-variables or likelihood method. At $m=8$, all inverse-based
controls use minimum-norm proxies, not identifiable instantaneous source
states; the scientifically relevant injective comparison is at $m=12,16$.

\paragraph{Evaluation and validation.}
Each fit reconstructs 32 held-out distributed initial states from their
entire eight-step histories, as in Eq.~\eqref{eq:v6metrics}. These are mean
trajectories without hidden process innovations; test sensor noise has
physical-channel standard deviation $0.001$ and temporal coefficient $0.4$.
This is offline reconstruction, not forecasting. Joint success means
$e_F,e_z\le0.10$. All 32 cases remain in each denominator; invalid fits have
infinite errors. The sweep contains 384 physical-response archives and
8,064 estimates (seven methods, three sensor arms). A separate validator
recomputed all metrics and fits, checked archive and code hashes, replayed
96 acquisitions, and verified 11,520 sensor projections, 192 exact audits,
96 ambiguity witnesses and 512 fewer-intervention checks. The protocol was
fixed before the sweep; all outcomes are retained.

\subsection{Effects of invertibility, excitation, and noise}
\label{app:v6injective-results}

\paragraph{Geometry and exact controls.}
At $m=12$, the four condition numbers of $L$ are 135.08, 532.09, 37.27 and
58.34; at $m=16$, they are 25.15, 44.72, 25.03 and 27.43. Median
instantaneous readout errors are 56.08\%, 2.19\% and 0.51\% at $m=8,12,16$,
respectively. The first is a noninjective minimum-norm proxy; the other two
are noisy injective readouts. These are not dynamics errors and do not
depend on the training budget.

With exact means and full reachability, baseline OLS recovers $F$ without
interventions for all 32 systems at both $m=12$ and $16$; maximum $e_F$ is
$2.83\times10^{-13}$. In the partial regime, the stable alternative
$F_{\rm alt}=[A,-C;0,-B]$ differs from $F=[A,C;0,B]$ by 157--174\% in
relative Frobenius norm, yet has identical baseline means at every lag by
block structure (zero numerical difference at all 24 audited lags).
Adding the eight intervention modes allows exact inverse-regression
recovery at $m=12,16$. Direct contrast reconstruction is exact in all 192
system/input/sensor audits, including $m=8$, with maximum
$e_F=6.02\times10^{-15}$.

\begin{table}[ht]
 \centering\small
 \setlength{\tabcolsep}{4pt}
 \caption{Square leadfield ($m=q=12$), lower noise, 18,432 trial equivalents.
 Errors are medians in percent; success requires both errors at most 10\%.
 Each input regime contains the same 32 anatomy--seed cases. OLS and GCV
 have the meanings defined above. ``Half'' means shared baseline only;
 ``all'' denotes the separately acquired equal-total-budget control.}
 \label{tab:v6injective-primary}
 \begin{tabular}{@{}lrrrrrr@{}}
 \toprule
 &\multicolumn{3}{c}{Partial reachability}&\multicolumn{3}{c}{Full reachability}\\
 \cmidrule(lr){2-4}\cmidrule(l){5-7}
 Method & $e_F$ & $e_z$ & Success & $e_F$ & $e_z$ & Success\\
 \midrule
 Direct & 1.62 & 1.78 & 32/32 & 2.45 & 3.06 & 32/32\\
 Inverse + contrasts, OLS & 18.69 & 74.40 & 0/32 & 23.34 & 92.29 & 0/32\\
 Inverse + contrasts, GCV & 18.63 & 75.40 & 0/32 & 23.43 & 95.94 & 0/32\\
 Baseline half, OLS & 704.53 & 363.71 & 0/32 & 27.10 & 42.62 & 3/32\\
 Baseline half, GCV & 81.91 & 195.09 & 0/32 & 24.06 & 31.74 & 3/32\\
 Baseline all, OLS & 704.41 & 363.70 & 0/32 & 17.82 & 25.17 & 7/32\\
 Baseline all, GCV & 82.08 & 199.86 & 0/32 & 18.31 & 24.98 & 7/32\\
 \bottomrule
 \end{tabular}
\end{table}

\paragraph{Noisy square recovery.}
Table~\ref{tab:v6injective-primary} shows that invertibility alone does not
make estimation easy. Direct reconstruction succeeds in all 32 cases of the
primary square cell under each input design. With partial reachability,
OLS and ridge--GCV each give 0/32 successes for all three data allocations:
the same-response interventional control, shared half-budget baseline,
and separately acquired all-budget baseline. With full reachability,
the interventional controls still give 0/32; baseline controls give 3/32
at half budget and 7/32 at full budget, with either OLS or ridge--GCV.
In the partial regime, large baseline OLS errors reflect
fitting unexcited directions from noise after ill-conditioned readout;
ridge controls the growth but cannot supply absent mean-response
information.

The direct method uses the rank-one structure of the entire contrast
history before calibrating and shifting it, whereas the controls first
invert each sensor block. This distinction is consistent with the observed
advantage, although the comparison does not isolate the contributions of
denoising, weighting and regressor error. The inverse-readout controls do not
regularize the leadfield inversion or jointly model correlated response
noise.

\begin{figure}[ht]
 \centering
 \includegraphics[width=\linewidth]{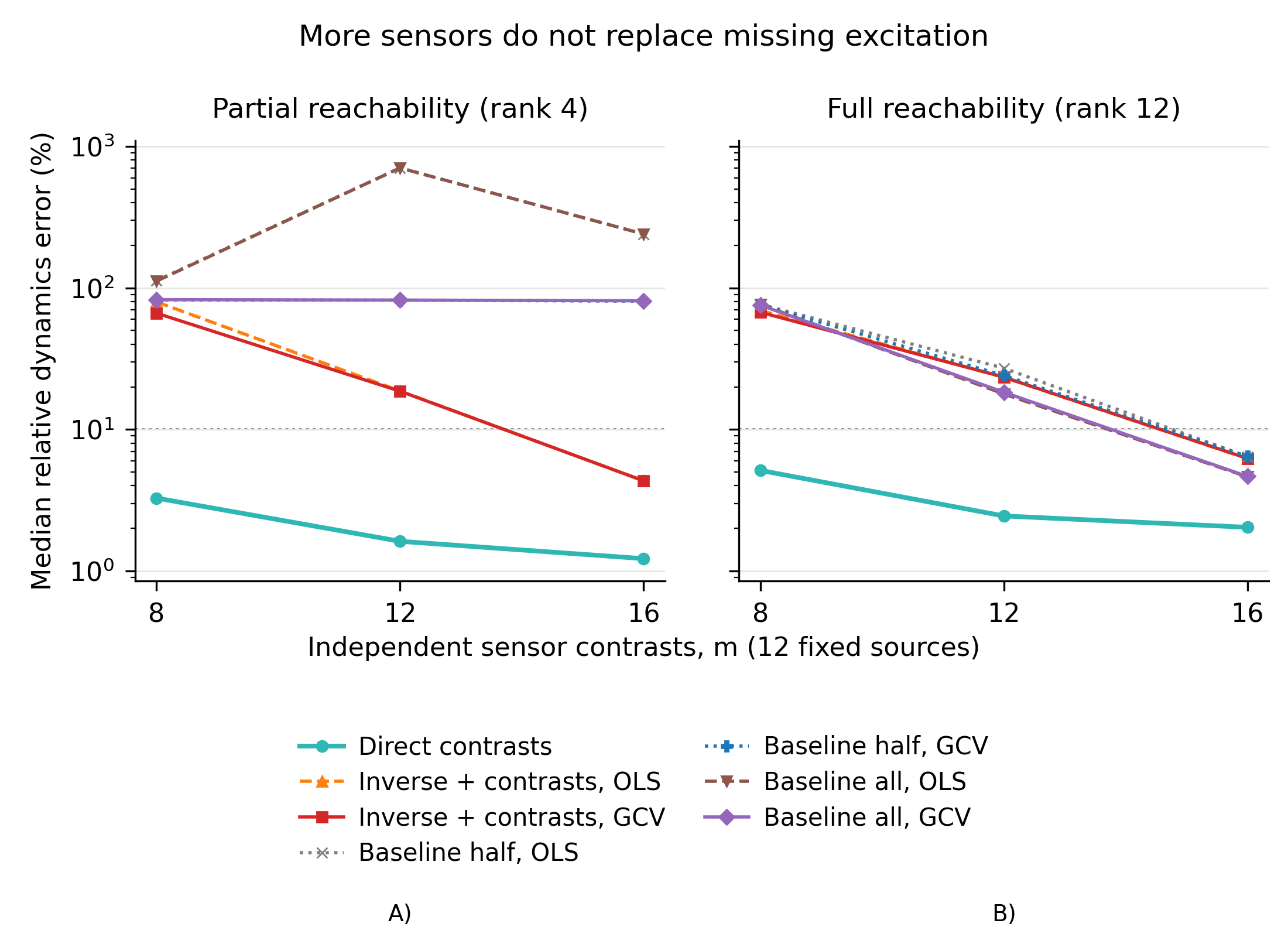}
 \caption{All seven methods across nested sensor counts at lower noise and
 18,432 trial equivalents. A)~Partial baseline reachability (rank four).
 B)~Full baseline reachability (rank twelve).
 Points are median dynamics errors over 32 cases; the
 vertical scale is logarithmic. Inverse-based values at $m=8$ are
 noninjective minimum-norm diagnostics. The horizontal 10\% threshold
 concerns $e_F$ only; joint success also requires $e_z\le10\%$.
 All underlying responses are simulated on MRI-derived anatomy.}
 \label{fig:v6injective-sensors}
\end{figure}

\begin{table}[ht]
 \centering\small
 \setlength{\tabcolsep}{5pt}
 \caption{Direct contrast reconstruction: sensor and noise sensitivities.
 All errors are median percentages over 32 cases. The upper block fixes
 lower noise and 18,432 trial equivalents; the lower block fixes $m=12$ and higher
 noise and lists total budgets in trial equivalents. Rows are paired acquisitions,
 not independent replications.}
 \label{tab:v6injective-sensitivity}
 \begin{tabular}{@{}lrrrrrr@{}}
 \toprule
 &\multicolumn{3}{c}{Partial reachability}&\multicolumn{3}{c}{Full reachability}\\
 \cmidrule(lr){2-4}\cmidrule(l){5-7}
 Condition & $e_F$ & $e_z$ & Success & $e_F$ & $e_z$ & Success\\
 \midrule
 $m=8$, lower noise & 3.27 & 3.90 & 31/32 & 5.14 & 6.68 & 30/32\\
 $m=12$, lower noise & 1.62 & 1.78 & 32/32 & 2.45 & 3.06 & 32/32\\
 $m=16$, lower noise & 1.22 & 1.33 & 32/32 & 2.03 & 2.22 & 32/32\\
 \midrule
 4,608, higher noise & 13.22 & 14.42 & 2/32 & 21.03 & 24.05 & 0/32\\
 18,432, higher noise & 6.56 & 7.16 & 28/32 & 9.82 & 12.48 & 9/32\\
 73,728, higher noise & 3.25 & 3.57 & 32/32 & 4.94 & 6.21 & 32/32\\
 \bottomrule
 \end{tabular}
\end{table}

\paragraph{Excitation and noise sensitivities.}
At $m=16$, full reachability, lower noise and an all-baseline budget of
73,728 trial equivalents,
ordinary OLS achieves 32/32 joint recoveries with median $e_F=2.30\%$ and
$e_z=2.30\%$; GCV achieves 31/32. Thus baseline observations can be
sufficient in finite noise, not merely in an exact algebraic check.
Conversely, partial-input all-budget baseline GCV achieves 0/32 with median
$e_F=80.28\%$ in that same sensor/noise/budget cell. Interventional OLS and
GCV both achieve 32/32 in both input regimes there. Higher-noise square
results in Table~\ref{tab:v6injective-sensitivity} demonstrate failure
despite injectivity and complete direct coverage. Of all 8,064 fits, two
are invalid: direct reconstruction in the full-input, $m=8$, higher-noise
cell with 4,608 trial equivalents. They remain in the reported error and success summaries.

\subsection{Exact propagation with fewer intervention modes}
\label{app:v6injective-fewer}

This separate \emph{noiseless} diagnostic applies the inverse-regression
route in Eq.~\eqref{eq:v6injective-propagation}, not the direct
column-completion estimator. For the same 32 partial-input systems, retain
the first 1, 2, 4 or 8 intervention modes in their fixed seeded order and
use $T=s=6$ or 12. No target or horizon is selected from recovery outcomes.
Table~\ref{tab:v6injective-fewer} gives the square results; the $m=16$ arm
has the same ranks and threshold counts.

\begin{table}[ht]
 \centering\small
 \caption{Exact-mean fewer-mode boundary, $m=q=12$. ``Recovered'' means
 $e_F<10^{-8}$, a numerical threshold distinct from noisy joint success.
 Ranks are common to all 32 systems in each row. There is no added noise.}
 \label{tab:v6injective-fewer}
 \begin{tabular}{@{}rrrrrr@{}}
 \toprule
 $T=s$ & Modes & $\rank\Omega$ & $\rank W_-$ & Recovered & Maximum $e_F$\\
 \midrule
 6 & 1 & 5 & 9 & 0/32 & $5.51\times10^{-1}$\\
 6 & 2 & 6 & 12 & 32/32 & $1.51\times10^{-11}$\\
 6 & 4 & 8 & 12 & 32/32 & $9.80\times10^{-14}$\\
 6 & 8 & 12 & 12 & 32/32 & $1.12\times10^{-14}$\\
 12 & 1 & 5 & 12 & 31/32 & $2.90\times10^{-7}$\\
 12 & 2 & 6 & 12 & 32/32 & $1.25\times10^{-12}$\\
 \bottomrule
 \end{tabular}
\end{table}

For one rank-one intervention and $T=6$, all new predecessor columns lie
in $\operatorname{span}\{e_j,Fe_j,\ldots,F^4e_j\}$. This adds at most five
directions to the four-dimensional baseline space, so
$\rank W_-\le9$, agreeing with the observed rank. Two modes give full
predecessor rank in all cases despite $\rank\Omega=6<12$, and maximum
$e_F=1.51\times10^{-11}$. With one mode and $T=12$, the predecessor matrix
has numerical rank 12 but may be extremely ill conditioned:
$\sigma_{\min}(W_-)$ reaches $4.75\times10^{-10}$. One square-case fit
exceeds the strict $10^{-8}$ error threshold even without added noise.
The $m=16$ maximum error in that row is $5.88\times10^{-8}$, also giving
31/32 below threshold. Both four- and eight-mode horizon-12 arms recover
32/32 at both sensor counts.

Response propagation can therefore identify dynamics with fewer targets
than column completion requires in the injective, exact-mean setting.
Practical design must also account for acquisition duration and conditioning:
longer horizons add recording cost, and nearly dependent propagated
directions amplify noise. The noisy sweep above uses all eight modes; the
fewer-mode diagnostic establishes the algebraic possibility, rather than a
finite-noise sample-efficiency result.

\section{Exploratory comparison with known-leadfield state-space fitting}
\label{app:v6em}

\subsection{Question and scope}

A known forward model also permits iterative estimation of hidden source
states and their dynamics. State-space approaches to cortical connectivity
include \citet{cheung2010statespace} and Network Localized Granger Causality
(NLGC; \citealp{soleimani2022nlgc}). We ask whether the direct response
construction remains useful alongside this estimation route, particularly as
an initialization when baseline controlled means excite only part of the
source space. Our comparator is a calibrated known-$L$, first-order Gaussian
state-space implementation, not a reproduction of either published method.
It does not include NLGC's sparsity or connectivity-testing procedure.
This is an exploratory optimization comparison, separate from the
160-case matched-information spectral-local study.

\subsection{Paired individual-trial acquisition}

We reuse two dynamics seeds (6100 and 6101) on each of the four Localize-MI
anatomies, giving eight system--anatomy combinations, with $q=12$, $m=8$,
$K=[e_1,e_2,e_3,e_4]$ and $T=s=6$. The generator and general receiving-row
changes follow Appendix~\ref{app:v6partial}: baseline controlled means reach
four coordinates, and eight exposed local changes target the remaining
coordinates. Estimators receive neither the block structure of $F$, true
targets nor changed coefficients. Two noise levels and two budgets give
32 paired acquisitions, not 32 independent systems or a fresh confirmation
cohort. The EEG and dynamics are simulated; only the forward anatomy comes
from the dataset.

Unlike the direct sampling of trial means in the principal studies, this
comparison generates fresh individual simulated episodes. The direct estimator forms
its usual response blocks from averages of those exact episodes. The
state-space comparator uses the baseline's unique times $0,\ldots,11$ and
each active episode's six post-transition observations
$\tau_{\mathrm{int}}+1,\ldots,\tau_{\mathrm{int}}+6$.
These are the temporal readouts underlying the response blocks; repeated
Hankel entries are counted only once in the likelihood. The changed
transition is excluded from the fit of baseline $F$. Unknown post-window
initial states absorb its effect. Neither fit uses the discarded
pre-intervention readouts of active episodes.

With $n=48$ or $192$ repeats per active condition and $48n$ per baseline
initialization, the total is $384n=18{,}432$ or $73{,}728$ episodes, half
baseline. Budgets use nested prefixes of raw trials, and noise levels reuse
standard random draws with different amplitudes. At lower noise, the
single-trial sensor and process standard deviations are $0.01$ and $0.002$;
higher noise multiplies both by four. Process innovations are independent
Gaussian with covariance $\sigma_p^2I_q$. Sensor noise is stationary Gaussian
with spatial covariance entries $\sigma_s^2 0.3^{|i-j|}$ and temporal
autoregressive coefficient $0.4$. Source and sensor noise streams are
independent.

\subsection{Comparator, information access and optimization}

\paragraph{What is fitted?}
The comparator alternates estimation of latent trajectories by Kalman
filtering and Rauch--Tung--Striebel smoothing with an unrestricted update
of $F$. The latent state is augmented with sensor noise, so the known
temporal correlation is represented explicitly. Baseline source initial
states are the known columns of $K$. Each intervention/insertion window
instead has unknown initial source means, one per input, and an unknown
full initial covariance shared across inputs. At every candidate $F$, we
maximize the Gaussian observed-data likelihood over these nuisance initial
distributions using the entire window before the smoothing and transition
update. Means are fitted by generalized least squares. For covariance fitting,
we whiten observations by their conditional noise covariance, project
onto the initial-state observation subspace, retain the positive-semidefinite
part of the projected covariance excess over identity noise, and map
back to source coordinates. This profiled expectation/conditional-maximization procedure
is abbreviated EM below. It avoids fixing unobserved initial coordinates
by applying a sensor pseudoinverse to only the first sample.

\paragraph{What information differs?}
Both methods use the same acquisition and known $L,K$, but their assumptions
are not identical. EM receives the exact process and sensor noise parameters
and uses within-condition trial covariances as well as means. It does not
use receiving-row locality, shared intervention coefficients or a sparsity
prior. Direct estimation uses locality and matched mean contrasts, without
calibrated noise parameters. Consequently, this is not a matched-prior
comparison or a test of whether all information in passive EEG is
insufficient: stochastic innovations can supply information absent from
controlled means.

\paragraph{Starts and stopping rules.}
Standalone EM uses three non-direct starts: source pseudoinversion followed
by first-order vector autoregression, and two fixed-seed perturbations
of $0.4I_q$ with entrywise standard deviation $0.15/\sqrt q$.
Their seeds are 202609200 and 202609201. Before fitting, spectral radii
above $0.95$ are rescaled to $0.95$; starts with condition number above
$10^6$ then receive fixed-seed Gaussian jitter with entrywise standard
deviation $0.02/\sqrt q$ (seed 202609202). Subsequent transition updates
have no stability clipping or sparsity penalty. Restart selection uses
only final training likelihood, never true $F$ or held-out outcomes.
A separate fit starts from the direct estimate under the same rules;
it is a refinement, not an independent competitor. Each run is capped
at 500 updates. Convergence requires relative likelihood improvement at
most $10^{-7}$ and relative Frobenius change of $F$ at most
$\sqrt{10^{-7}}$, allowing likelihood decreases only within a
$10^{-9}$ relative numerical tolerance.

\subsection{Recovery results and optimization diagnostics}

Evaluation follows Equation~\eqref{eq:v6metrics}: the source-history metric
uses all eight held-out observations to reconstruct their source states,
not to forecast eight unobserved steps. Table~\ref{tab:v6em} reports every
acquisition. Joint success requires both $e_F$ and $e_z$ to be at most
$0.10$. No failed or iteration-capped case is removed.

\begin{table}[htbp]
\centering\small
\setlength{\tabcolsep}{5pt}
\caption{\textbf{Exploratory known-leadfield fitting under partial excitation.}
Each method cell gives median $e_F$ (percent) and joint successes out of
eight systems. The same systems recur in all four rows. Standalone EM is
selected by training likelihood among three non-direct starts; direct + EM
starts from our estimate. Selected standalone fits converge in 0/8, 0/8,
1/8 and 2/8 cases in row order, versus 8/8 in every direct + EM row.
Capped fits remain included; these are not globally optimized likelihood
comparisons.}
\label{tab:v6em}
\begin{tabular}{lrccc}
\toprule
Noise & Episodes & Direct & Standalone EM & Direct + EM\\
\midrule
Lower &18,432&2.99\%; 8/8&53.78\%; 0/8&2.95\%; 8/8\\
Lower &73,728&1.57\%; 8/8&53.66\%; 0/8&1.57\%; 8/8\\
Higher&18,432&12.19\%; 1/8&46.22\%; 0/8&11.96\%; 1/8\\
Higher&73,728&6.40\%; 5/8&47.62\%; 0/8&6.36\%; 6/8\\
\bottomrule
\end{tabular}
\end{table}

Direct estimation identifies all eight targets in every acquisition.
Its higher-noise failures therefore concern recovery accuracy rather than
labeling. Across 32 acquisitions, direct and direct + EM achieve 22 and
23 joint successes, respectively; standalone EM achieves none.
All 32 direct-initialized fits converge in 3--71 updates (median 16),
with higher training likelihood than each of the three non-direct fits.
Only 7/96 non-direct restarts converge, including 3/32 selected fits.

Refinement lowers $e_F$ in all 32 cases, but its median absolute reduction
is only 0.025 percentage points, or a median relative reduction of 0.61\%.
It lowers $e_z$ in only 14/32 cases. In table order, median direct $e_z$
is 3.95\%, 1.73\%, 15.03\% and 7.03\%; after refinement it is
3.94\%, 1.75\%, 15.76\% and 7.15\%. Thus the result supports useful
initialization, not a substantial or uniform reconstruction improvement.
The additional joint success is one threshold crossing.

\paragraph{Longer optimization and a positive control.}
For the first seed on each anatomy at lower noise and 18,432 episodes,
we continue the likelihood-selected standalone fit from 500 to 2,000
updates. This subset is chosen by index, not recovery outcome. None of
the four continuations converges or succeeds; final $e_F$ ranges from
38.58\% to 74.20\%. Likelihood increases in all four, but anatomical
accuracy need not. Continuing only one selected start per system is a
bounded sensitivity check, not a global-optimality test.
In a separate positive control on the same four anatomies and first
dynamics seed, all twelve known initializations are available
($K=I_{12}$), with no interventions. At lower noise and 2,304 episodes
(192 per input), EM with the two generic starts specified above and a 1,000-update cap
achieves joint success in 4/4 cases, with $e_F$ from 0.95\% to 4.23\%.
This verifies recovery in a favorable regime, but changes the protocol
and is not a budget-matched comparison.

Ten implementation and acquisition tests pass. A separate dense-Gaussian
calculation reproduces all 128 final restart likelihoods, with maximum
absolute discrepancy $7.79\times10^{-7}$ on million-scale likelihoods;
it also checks response blocks, budgets, metrics and restart selection.
The positive controls and longer-run diagnostics pass the same likelihood
check. The evidence concerns this calibrated implementation and finite
optimization budget. It neither establishes a general failure of
state-space estimation nor validates recovery from recorded neural EEG.
Fresh-system confirmation and stronger or intervention-aware likelihood
optimization remain necessary before treating the accuracy gap as a
general comparative claim.

\section{Signed initialization under additive insertion effects}
\label{app:v6signed}

This separate simulation asks whether controlled initialization amplitudes
can preserve the local-row contrast in the presence of a simultaneous
additive effect. It uses the same direct estimator after an observed-data
transformation. Its additional initialization capability is not required
by the identification theorem or the incomplete-reachability experiment.

\subsection{Cancellation and its assumptions}

Let the known initial conditional mean be $aK_{:,i}$, where $a$ is a
controlled scalar. Suppose the insertion transition has conditional mean
\[
 \bar z_{\tau_{\mathrm{int}}+1}
 =(F+D_e)\bar z_{\tau_{\mathrm{int}}}+c_e,
 \qquad D_e=e_{j_e}v_e^\top,
\]
where the unknown additive increment $c_e\in\R^q$ is fixed across
initialization amplitudes, directions, and insertion times for that mode.
The same $F,L,K,D_e$ apply at both signs, and all other transitions use $F$.
Using the histories $O_T,R_s$ of Eq.~\eqref{eq:v6histories}, the response
blocks at amplitude $a$ satisfy
\begin{align}
 H_0(a)&=aO_TR_s,\qquad H_+(a)=aO_TFR_s,\nonumber\\
 H^{[e]}(a)&=aO_T(F+D_e)R_s+O_Tc_e\mathbf1_{rs}^{\top},\nonumber\\
 \Delta H_e(a)&=aO_TD_eR_s+O_Tc_e\mathbf1_{rs}^{\top},
 \label{eq:v6signed-affine}
\end{align}
where $\mathbf1_{rs}$ contains $rs$ ones and
$\Delta H_e(a)=H^{[e]}(a)-H_+(a)$. For $a\ne0$, the odd and even parts are
\begin{align}
 \frac{\Delta H_e(a)-\Delta H_e(-a)}{2a}
   &=(O_Te_{j_e})(v_e^\top R_s),\nonumber\\
 \frac{\Delta H_e(a)+\Delta H_e(-a)}2
   &=O_Tc_e\mathbf1_{rs}^{\top}.
 \label{eq:v6signed-parity}
\end{align}
These identities follow directly because the initialized mean state is
linear in $a$ while $c_e$ is constant. More generally, the slope
$\mathcal S(X)=[X(a_2)-X(a_1)]/(a_2-a_1)$, for $a_2\ne a_1$, removes
the additive term. Applying it to \emph{all} baseline, shifted-baseline,
and active blocks restores the original $K$ and contrast factorization.
The anatomical calibration and completion then use the same
$\Omega=[K,E_J]$ and require the original exposure, coverage, and temporal
observability assumptions.

This is classical affine cancellation under a causal premise. Reversing
stimulation polarity does not by itself reverse the known source-state
initialization. Additive effects that depend on amplitude, including odd
artifacts, can survive the slope; sign-dependent mechanisms, saturation,
or initial-state distributions that break the affine mean model invalidate
Eq.~\eqref{eq:v6signed-affine}. Rank one or a leadfield match alone therefore
does not certify a mechanism change.

\subsection{Fixed simulation and acquisition cost}

The study uses twelve sources, eight sensor contrasts, four known
initialization directions, eight hidden intervention targets, and $T=s=6$.
The dynamics, gain-suppression family, trial-noise law, and offline held-out
evaluation follow Appendix~\ref{app:v6experiments}. Each mode adds an
independent diffuse $c_e$ of norm $0.1$, a fixed stress condition in normalized
units rather than a measured biological effect size. Forty fresh seed indices
2000--2039 are shared across four geometries, three budgets, and three designs:
480 anatomy--seed--budget cases and 1,440 fitted-design rows. The protocol
was fixed before outcomes. Seeds reused across geometries do not supply
160 independent biological observations.

For $n\in\{64,256,1024\}$, each design uses $N_{\rm tot}=384n$ trial
equivalents, half baseline, and thirteen samples per episode. Positive-only
uses amplitude $+1$; the two slope designs use $0/+1$ or $-1/+1$, with
half the trial equivalents at each amplitude. They match represented trial
count, episode duration, and peak amplitude. Their squared-initialization-amplitude proxies are
$N_{\rm tot}$ for positive-only and balanced signed, but $N_{\rm tot}/2$
for zero/positive. The signed comparison is therefore not an equal-energy
advantage, and these proxies are not calibrated joules. With independent,
equal-covariance amplitude episodes, the $0/+1$ slope has four times the
noise variance of the $-1/+1$ slope. Matched Gaussian streams make signed
probing at $n$ and zero/positive probing at $4n$ numerically equivalent;
these repeated results are a design check, not independent evidence of a
fourfold biological efficiency gain.

Only measured transformed blocks and known $L,K$ enter the direct
estimator; it receives no true $F,D_e,c_e$, targets, or held-out sources.
The positive-only design is deliberately uncorrected for the additive effect.
Held-out evaluation estimates the initial states from eight sensor
observations for each of 32 independent trajectories. The errors $e_F,e_z$
and joint success criterion $e_F\leq0.1$, $e_z\leq0.1$ are those of
Eq.~\eqref{eq:v6metrics}.

\begin{table}[htbp]
\centering\small
\setlength{\tabcolsep}{4pt}
\caption{Complete signed-initialization confirmation. Each row contains
160 cases. $N_{\rm tot}$ counts trial equivalents. Errors are percentages;
medians and nearest-order-statistic p90
include invalid estimates as $\infty$. Designs and budgets are paired.}
\label{tab:v6signed}
\begin{tabular}{@{}rlrrrrr@{}}
\toprule
$N_{\rm tot}$ & Design & Valid & Success & Median $e_F$ & Median $e_z$ & p90 $e_z$\\
\midrule
24,576 & Positive-only &0&0&$\infty$&$\infty$&$\infty$\\
24,576 & $0/+1$ slope &155&77&10.371&8.265&22.731\\
24,576 & $-1/+1$ slope &158&120&5.165&4.144&10.695\\
\midrule
98,304 & Positive-only &0&0&$\infty$&$\infty$&$\infty$\\
98,304 & $0/+1$ slope &158&120&5.165&4.144&10.695\\
98,304 & $-1/+1$ slope &160&150&2.553&2.062&5.401\\
\midrule
393,216 & Positive-only &0&0&$\infty$&$\infty$&$\infty$\\
393,216 & $0/+1$ slope &160&150&2.553&2.062&5.401\\
393,216 & $-1/+1$ slope &160&158&1.276&1.031&2.657\\
\bottomrule
\end{tabular}
\end{table}

At the primary budget of 98,304 trial equivalents, balanced signed initialization gives
150/160 successful recoveries, versus 120/160 for zero/positive. All balanced
target labels are correct, but ten finite estimates still miss the joint
error threshold. Positive-only fails inferred coordinate coverage in all
480 budget-specific cases; the slope designs have nine additional invalid
fits. Increasing repeats cannot remove the fixed additive confound from
the uncorrected contrast. This experiment establishes neither a biological
signed-initialization protocol nor robustness to amplitude-dependent effects.

An independent audit checked all 480 saved model files and the
full 1,440-row grid, regenerated every amplitude and transformed block,
reproduced all fits bitwise, and verified source, geometry, and result hashes.
Separately computed metrics and summaries agreed within $1.78\times10^{-15}$.

\section{Boundary control: fewer targets under an exact-gain prior}
\label{app:v6gain-boundary}

The general-row control in Appendix~\ref{app:v6arbitrary} belongs to a
separate study, specified before outcomes, whose primary question was whether a stronger
intervention law permits fewer distinct targets at the same trial cost.
This comparison delineates the role of the additional intervention assumption.
The auxiliary gain hybrid assumes the exact structural law
\begin{equation}
 D_e=-\eta_e e_{j_e}e_{j_e}^{\top}F,\qquad \eta_e\neq0.
 \label{app:v6eq:gain-boundary-law}
\end{equation}
The numerical gains $\eta_e$ and realized target labels $j_e$ remain unknown
to fitting. This additional relation between an intervention and the baseline
row is not assumed by direct general-row completion.

\paragraph{Auxiliary completion and its assumptions.}
Suppose $K=E_{\mathcal I}$ consists of known coordinate inputs, and the
distinct exposed target set $J$ is disjoint from $\mathcal I$. The same
anatomical labeling and scale conditions as in Appendix~\ref{app:v6setup}
identify $(O_T)_{:,j_e}$ from
$\Delta H_e=(O_T)_{:,j_e}\psi_e^\top$. Under
Eq.~\eqref{app:v6eq:gain-boundary-law},
$\psi_e^\top=-\eta_e e_{j_e}^\top FR_s$.
Let $W\in\R^{rs\times d_W}$ span the joint right nullspace of these
profiles, put $U=\{1,\ldots,q\}\setminus J$ and
$M=U\setminus\mathcal I$, and define $X=FR_sW$. Then
\begin{equation}
 X_{J,:}=0,\qquad H_+W=(O_T)_{:,U}X_{U,:},\qquad
 X_{U,:}=L_{:,U}^{\dagger}(H_+W)_{\rm top},
 \label{app:v6eq:gain-boundary-nullspace}
\end{equation}
where ``top'' selects the first $m$ rows, and the last identity requires
$L_{:,U}$ to have full column rank. The input columns
$(O_T)_{:,\mathcal I}$ are already observed in the first $r$ columns of
$H_0=O_TR_s$. If $X_{M,:}$ has full row rank, the remaining columns are
\begin{equation}
 (O_T)_{:,M}
 =\bigl[H_+W-(O_T)_{:,\mathcal I}X_{\mathcal I,:}\bigr]X_{M,:}^{\dagger}.
 \label{app:v6eq:gain-boundary-completion}
\end{equation}
After assembly, the known top block is set to $L$, and temporal least squares
recovers $F$ when $O_-$ has full column rank. These are sufficient population
conditions; noisy fits use estimated profiles and columns. For an arbitrary
row, $\psi_e^\top=v_e^\top R_s$ only implies $v_e^\top R_sW=0$, not
$e_{j_e}^\top FR_sW=0$, so the completion identity need not hold.
If every non-input coordinate is targeted, $M$ is empty and the hybrid
reduces to direct completion with $\Omega=[K,E_J]$, regardless of the gain law.

\paragraph{Paired comparison and trial accounting.}
The study uses $q=12$, $m=8$, $r=4$, $T=s=6$, and forty fresh indices
5000--5039 independently namespaced within each of four fixed geometries:
160 simulated baseline systems, not 160 anatomies. Two row classes, two
target counts, and four estimators give 640 acquisitions and 2,560
method-specific results. The four-target set is chosen from known $L,K$ by maximizing
$\sigma_{\min}(L_{:,U})$ among non-input coordinate subsets; eight targets
cover all non-input coordinates. Every acquisition pays 49,152 baseline and
49,152 active trial equivalents, totaling 98,304. Four targets receive 512
repeats per input/target/insertion condition; eight receive 256. Fewer targets
therefore do not mean fewer trial equivalents. Conditions and designs are paired within
each system using matched noise draws; equal row norms across the gain and
isotropic arbitrary-row conditions do not imply equal response exposure.

\begin{table}[ht]
\centering\small
\caption{Selected comparisons from the boundary study, at 98,304 trial
equivalents each. Joint success requires both $e_F\leq0.1$ and $e_z\leq0.1$;
source errors are percentages. The arbitrary-row direct and spectral-local rows
are the same secondary control reported in Appendix~\ref{app:v6arbitrary}.}
\label{tab:v6gain-boundary}
\begin{tabular}{@{}llrrrr@{}}
\toprule
Row class & Estimator, targets & Valid & Success & Median $e_z$ & p90 $e_z$\\
\midrule
Exact gains & Hybrid, 4 &160/160&144/160&1.965&6.704\\
Exact gains & Direct, 8 &160/160&142/160&2.383&7.016\\
Arbitrary rows & Hybrid, 4 &160/160&0/160&310.997&1267.994\\
Arbitrary rows & Direct, 8 &160/160&157/160&2.106&4.157\\
Arbitrary rows & Spectral-local, 8 &153/160&139/160&2.915&9.989\\
\bottomrule
\end{tabular}
\end{table}

Under exact gains, the primary paired success difference, hybrid minus
direct, is $+1.25$ percentage points. Its 10,000-replicate,
40-index-block bootstrap 95\% interval is $[-3.125,5.625]$ percentage points,
retaining all four fixed geometries per block. This approximate simulation
interval establishes neither superiority nor, without a specified margin,
noninferiority or equivalence. The hybrid's maximum source error is
140.332\%, versus 22.381\% for direct: its favorable median does not remove
the adverse tail.

Under arbitrary-row changes, all 160 four-target hybrid fits are finite and
all target labels are correct, yet none succeeds; median source error is
310.997\% and the maximum is 4377.539\%. Numerical validity is therefore
not a calibrated model-adequacy test. Direct with four targets instead
abstains in all 160 cases in each row class because coordinate coverage
fails. At eight targets, the hybrid agrees with direct, including its
157/160 arbitrary-row successes; this is a limiting-case check, not an
independent method comparison. The study supports a conditional tradeoff
between target diversity and structural assumptions, not a universal
reduction in the targets required by the unrestricted-row model.

\section{Recorded-current EEG: spatial transfer and anatomical calibration}
\label{app:v6recorded}

These complementary Localize-MI analyses \citep{mikulan2020} examine whether
repeatable sensor responses support anatomical calibration. The measured
signal is predominantly the electrical artifact of injected current. Its
contact midpoint is a localization reference for that current, not ground
truth for subsequent neural generators. No simulated sources or noise are
added. Neither analysis estimates neural $F$, supplies a known neural
initialization map $K$, or implements a receiving-row mechanism change.
They reuse the four participants supplying the simulation anatomies; the
overlapping site sets are not independent biological replications.

\subsection{Data and anatomical diagnostic}
\label{app:v6recorded-data}

The released epochs contain 256 scalp channels sampled at 8 kHz from
$-250$ to $+10$ ms, preprocessed and artifact-aligned by the dataset authors.
The response window is $-2$ through $+2$ ms (33 samples); background
covariance is estimated from $-250$ through $-50$ ms. Each participant's
MRI-derived forward model contains 8,196 cortical positions with three
orientation components per position. The fits search the complete bilateral,
free-orientation dictionary, rather than the twelve fixed sources and
orientations used in the simulations. Data and gains are average-referenced
within each analysis's common good-channel set. Contact coordinates enter
only evaluation after fitting; they do not select candidates, orientations,
hemisphere, response windows or controls.

For a raw channel-by-time mean $\bar X_A$ over trial indices $A$, let $W$
be the background whitener and $C_j$ an orthonormal basis of the whitened
three-orientation gain at cortical position $j$. Conventional dipole fitting
selects
\begin{equation}
 \widehat j_A=\arg\max_{j=1,\ldots,8196}
 \frac{\Frob{C_j^\top W\bar X_A}^2}{\Frob{W\bar X_A}^2}.
 \label{eq:v6recorded-dipole}
\end{equation}
This standard diagnostic permits time-varying dipole moments in a subspace
of dimension at most three. It is neither a new localization algorithm nor
the proposed fixed-column anatomical-label estimator, and it does not force
the response matrix to have rank one.

\subsection{Cross-intensity spatial transfer}
\label{app:v6recorded-transfer}

Selection includes every site with two distinct current intensities, at
least 20 retained trials per run, and exactly one run in the previously
available data subset: 11 sites, 22 runs and 822 trials from participants 01,
03, 05 and 07 (Table~\ref{tab:v6recorded-sites}). Eleven companion runs were
downloaded after the analysis protocol was fixed, but previously available
runs and their localization results had already been examined. This is
developmental reuse with an outcome-independent new analysis plan, rather
than an untouched cohort or external preregistration.

Alternating trial indices define two complementary training/test folds.
Both low-to-high and high-to-low transfers give four tests per site, or
44 total. Disjoint indices do not guarantee temporal independence. Current
intensity was not randomized trial by trial, so differences between runs
cannot be identified as causal effects of amplitude. The different-site
control uses the lowest-numbered previously available run at another site
in the same participant. The intersection of published good-channel lists
from this run and both intensity runs gives 153--235 electrodes per pair.
All compared predictions share target trials, channels and reference.

For each pair and fold, only training prestimulus samples from the two
intensity runs estimate $\widehat\Sigma$. With $c$ common electrodes, use
\[
 \widehat\Sigma_{\rm reg}
 =0.9\widehat\Sigma+0.1\frac{\operatorname{tr}(\widehat\Sigma)}{c}I_c.
\]
Remove the average-reference null direction and construct $W$ from the
positive covariance eigenvalues. Test trials and the different-site control
do not contribute to this covariance. Write the whitened training and test
mean channel-by-time responses as $X_{\iota}^{\rm tr}$ and
$X_{\iota}^{\rm te}$. For the first $d$ left singular vectors
$B_{\iota,d}$ of the source-intensity training mean, predict at the other
intensity $\iota'$ using
\begin{equation}
 \widehat X_{\iota'}^{\rm te}
 =B_{\iota,d}B_{\iota,d}^{\top}X_{\iota'}^{\rm tr},
 \qquad
 \mathcal Q=1-
 \frac{\Frob{X_{\iota'}^{\rm te}-\widehat X_{\iota'}^{\rm te}}^2}
      {\Frob{X_{\iota'}^{\rm te}}^2}.
 \label{eq:v6recorded-transfer}
\end{equation}
The basis transfers from the source intensity; target training responses
fit the $d\times33$ temporal coefficients. This is spatial transfer with
target-condition calibration, not zero-shot waveform prediction. No test
waveform is fitted. The uncentered score $\mathcal Q$ compares with a zero
predictor: one is perfect prediction, zero is the zero predictor, and
negative scores are allowed. It is neither centered regression $R^2$ nor
localization accuracy.

Transferred ranks $d=1,2,3$ are compared with four controls: a rank-one basis
learned from target training trials; a dipole location selected from source
training data by Eq.~\eqref{eq:v6recorded-dipole}, then frozen while target
training data fit its orientation-subspace coefficients; a rank-one basis
from the prespecified other site, also calibrated on target training data;
and the full target training mean. The last is an unrestricted repeatability
reference, not a known upper bound. The other-site control is not matched
for intensity, signal-to-noise ratio, trial count or acquisition time, and
is not a nearest-site or randomized intervention control.

\begin{table}[ht]
\centering\small
\caption{All 11 paired sites in the recorded-current transfer study. Current
and trial counts are low/high intensity. Scores average both folds and
directions; ``Target'' is target-trained rank one. A prime belongs to the
dataset contact identifier. Sites reuse four participants.}
\label{tab:v6recorded-sites}
\begin{tabular}{llrrrrr}
\toprule
Participant & Site & Current (mA) & Trials & Transfer rank 1 & Target & Other site\\
\midrule
01 & K13--14 & 1/5 & 38/36 & .974 & .977 & .293\\
01 & N2--3 & 1/5 & 38/36 & .894 & .966 & .200\\
01 & S1--2 & 1/5 & 27/34 & .949 & .953 & .242\\
01 & S5--6 & 1/5 & 38/32 & .974 & .976 & .125\\
03 & R$'$2--3 & .3/.5 & 35/42 & .893 & .896 & .179\\
05 & G$'$8--9 & .1/.3 & 44/43 & .978 & .980 & .556\\
05 & H$'$2--3 & .1/.3 & 42/44 & .723 & .745 & .460\\
05 & X$'$2--3 & .1/.3 & 31/40 & .545 & .627 & .121\\
07 & B11--12 & .1/.3 & 32/38 & .925 & .934 & .005\\
07 & B$'$13--14 & .1/.3 & 46/39 & .969 & .971 & .002\\
07 & Q16--17 & .1/.3 & 25/42 & .976 & .978 & .003\\
\bottomrule
\end{tabular}
\end{table}

All 308 model predictions (44 tests times seven models) are finite, with
no excluded site or failed result. After averaging folds and directions
within site, median $\mathcal Q$ across sites is 0.949 for transferred rank
one, 0.966 for target-trained rank one, 0.975 and 0.981 for transferred ranks
two and three, 0.639 for the frozen dipole, 0.179 for the other-site basis,
and 0.983 for the full target training mean. Transferred rank one beats the
other-site and dipole controls at all 11 sites, but is worse than target
rank one and transferred ranks two and three at every site. Its median
paired advantage over the other-site basis is 0.707; its median paired
disadvantage to target rank one is 0.00295. These are paired differences,
not differences of marginal medians or population significance tests.

The median site-averaged projective angle between intensity-specific
training directions is $3.76^\circ$. Nevertheless, participant 05's
H$'$2--3 and X$'$2--3 have transfer scores of 0.723 and 0.545; at X$'$2--3,
high-to-low transfer scores 0.366 versus 0.724 in reverse. Participant 01's
N2--3 loses 0.0714 relative to the target-trained direction. Higher-rank
improvements concern these raw response matrices $X$, not the intervention
contrast $\Delta H_e$. Their rank therefore neither verifies nor refutes
the theorem's rank-one contrast factorization.

Conventional localization has median error 13.49 mm over the 44 training
fits. Locations agree between intensities in both folds at seven of eleven
sites, yet some stable assignments remain 13--15 mm from the contact
midpoint. The lower frozen-dipole score does not isolate its cause:
approximate tissue or registration models, distributed bipolar current,
mislocalized dipoles, additional response components and between-run changes
can all contribute. Empirical direction portability and anatomical
calibration are distinct requirements.

\subsection{Averaging, repeatability and anatomical accuracy}
\label{app:v6recorded-repeats}

A separate finite-pool analysis uses 16 previously studied sites/runs,
552 trials and 171--236 good channels per run from the same four participants.
The 11 transfer sites are a subset of these sites. All records were already
available. The response/background windows, reference and free-orientation
dictionary are as above. Reserve each run's last five epochs in file order
as a held-out reference; the remaining $M=16$--41 epochs form the acquisition
pool. File order does not establish independently verified acquisition time.
Twenty fixed random permutations per site generate nested averages of
$n=1,2,4,8,16$ distinct pool epochs; permutations overlap. Whitening uses
prestimulus data from the entire acquisition pool, excluding held-out epochs,
and is fixed across $n$. Thus $n$ is an averaging budget conditional on
additional noise calibration, not total recording cost. At participant 05's
H$'$8--9, the pool contains exactly 16 epochs, so all $n=16$ averages coincide.

Equation~\eqref{eq:v6recorded-dipole} gives 1,600 finite budget fits
($16$ sites $\times20$ permutations $\times5$ budgets); 80 full-pool,
held-out and order diagnostics are also finite. Three endpoints are
distinguished: distance from the contact midpoint, agreement with the
full-pool fitted peak, and the projective angle between the leading left
singular vector of each whitened mean and that of the five-epoch held-out
mean. The full-pool peak includes the sampled epochs and may itself be
wrong; the held-out direction is noisy, not ground truth.

\begin{table}[ht]
\centering\small
\caption{Averaging study across all 16 sites. Error and angle first take a
median over 20 overlapping draws within site, then a median or 90th
percentile across sites. Peak agreement is the mean within-site fraction
matching the full-pool peak, not correct localization. These descriptive
percentiles are not confidence intervals.}
\label{tab:v6recorded-repeats}
\begin{tabular}{rrrrr}
\toprule
$n$ & Median error (mm) & Error p90 (mm) & Held-out angle & Peak agreement\\
\midrule
1  & 13.68 & 33.24 & $4.72^\circ$ & 66.88\%\\
2  & 13.68 & 24.79 & $4.77^\circ$ & 75.63\%\\
4  & 13.68 & 24.28 & $3.76^\circ$ & 83.13\%\\
8  & 13.68 & 24.28 & $3.30^\circ$ & 89.69\%\\
16 & 13.68 & 24.28 & $2.93^\circ$ & 92.50\%\\
\bottomrule
\end{tabular}
\end{table}

From $n=1$ to $n=16$, held-out direction angles improve at 15 of 16 sites,
while the median localization error remains 13.68 mm. Twelve site error
medians are unchanged; four improve, one by only 0.009 mm. At $n=16$,
15 sites have zero median distance to their full-pool peak, but eleven
of those remain more than 10 mm from the contact midpoint. Participant 03's
H$'$1--2 stays approximately 31.56 mm away despite stable estimates.
Participant 05's X$'$2--3 improves from 62.67 to 21.76 mm: averaging can
remove a major error without producing accurate anatomy. The 10 mm
comparison is descriptive, not a clinical threshold.

Direction stability also has adverse cases. At $n=16$, participant 05's
H$'$2--3 and X$'$2--3 retain held-out angles of $77.83^\circ$ and
$85.48^\circ$. Participant 01's N2--3 worsens from $1.99^\circ$ to
$14.99^\circ$; its first-five and last-five fitted positions are 20.88 mm
apart. Background activity, drift, pulse variability and reference noise
are possible explanations, but this analysis cannot distinguish them.
A stable dipole-subspace maximum can coexist with an unstable leading
empirical direction. A maximum of 16 averaged responses does not establish
an asymptotic localization-error floor.

\subsection{Reproducibility and interpretation}
\label{app:v6recorded-audit}

Protocols, selection lists, source checksums, splits, predictions and adverse
outcomes are retained in the research package. Independent transfer
validation reproduced all 308 predictions, with maximum predictive-score
discrepancy $1.64\times10^{-15}$. The averaging implementation's projected-Gram
shortcut was independently checked against 393,408 conventional dictionary
scores, with maximum discrepancy $8.88\times10^{-16}$.

These descriptive studies separate transfer of a measured sensor direction,
repeatability under averaging, and anatomical agreement. They provide
neither population-level inference from four reused participants nor a test
of neural dynamics recovery or the physical intervention contract.
The diagnostic dipole model and current artifact also differ from the fixed
source dictionary assumed by the theorem. The findings motivate explicit
forward-model mismatch assessment before treating a repeatable measured
direction as an anatomically calibrated response history; they do not
identify which forward-model error is responsible or establish clinical
localization performance.

\section{Physical forward-model mismatch and anatomical calibration}
\label{app:v6physical-readout}

This appendix describes the physical sensitivity study summarized
in Section~\ref{sec:v6physical-readout}. Its question is whether correct target
labels suffice for source-coordinate recovery when the supplied leadfield is
inaccurate. Dynamics and responses are simulated; only anatomical geometries
come from participants. The study changes a physical head-model parameter,
not arbitrary matrix entries, and does not estimate any participant's actual
conductivity.

\subsection{Paired systems and physical forward models}

The study reuses Localize-MI geometries 01, 03, 05 and 07
\citep{mikulan2020}, with a fixed selection of twelve cortical-normal
sources, nine scalp electrodes and eight independent reference contrasts.
Four simulated systems per anatomy give sixteen anatomy--seed cases, using
seed indices 7100--7103. Each is evaluated in five assumed physical worlds
with skull conductivity
\[
 c\in\{0.003,0.0045,0.006,0.009,0.012\}\ \mathrm{S/m}.
\]
A three-layer boundary-element model (BEM), using ico3 surface discretization,
keeps brain and scalp conductivities at $0.3\ \mathrm{S/m}$. The nominal
world is $c_0=0.006\ \mathrm{S/m}$. This is the reference for this controlled
sweep, not a claim about actual conductivities or the original dataset's
forward-model parameter choice.

The anatomical positions, sensor order, reference basis and source
orientations are fixed across worlds. Free-orientation gains are contracted
with cortical normals in head coordinates before selecting sensors and
applying the reference contrasts. Reconstructing the original fixed gain
from the released Cartesian gain and saved normals gives maximum relative
discrepancy $2.54\times10^{-8}$; the selected/referenced discrepancy is at
most $2.90\times10^{-8}$, consistent with stored single-precision gains.

\paragraph{Fixed units across physical worlds.}
Let $G_c\in\R^{8\times12}$ be the orientation-contracted, referenced physical
gain at conductivity $c$. Define one diagonal source-unit normalization,
\begin{equation}
 \Lambda_{\rm unit}
 =\diag\bigl(\|G_{c_0,:,1}\|_2,\ldots,\|G_{c_0,:,12}\|_2\bigr),
 \qquad L_c=G_c\Lambda_{\rm unit}^{-1},
 \qquad L_{\rm nom}=L_{c_0}.
 \label{eq:v6physical-units}
\end{equation}
Every world uses the same $\Lambda_{\rm unit}$. Independently normalizing
each world's columns would change the source units and hide forward-amplitude
errors that matter for anatomical calibration. The relative Frobenius
discrepancy $\Frob{L_c-L_{\rm nom}}/\Frob{L_{\rm nom}}$ ranges across
anatomies from 26.64--30.99\% at $c=0.003$, 11.17--13.45\% at $0.0045$,
15.73--20.03\% at $0.009$, and 26.80--34.74\% at $0.012$.
These ranges describe this assumed BEM family, not a measured distribution
of real leadfield errors.

At $c=0.0045$, anatomy-wise median projective direction changes are
$1.84$--$2.14^\circ$, with median signed projected gain ratios
$0.869$--$0.900$. At $c=0.009$, the corresponding ranges are
$2.62$--$3.42^\circ$ and $1.124$--$1.179$. Every exact true column remains
closest to the correct nominal leadfield line. Thus the sensitivity study
does not rely on unrelated or permuted readout maps. Direction and amplitude
change together; it is not a scale-only ablation.

\subsection{Acquisition, fitting and evaluation}

The baseline generator is the twelve-source generator described in
Appendix~\ref{app:v6data}, with separate random streams for anatomy--seed
cases. The initialization map is $K=[e_1,e_2,e_3,e_4]$. Eight targets are
the remaining source coordinates, in a hidden random order, with
$D_e=-\eta_e e_{j_e}e_{j_e}^{\top}F$ and
$\eta_e\sim\operatorname{Unif}[0.2,0.6]$. These target labels and
coefficients are never provided to the blind estimator. Each system's
$F,K,D_e$, held-out initial states and noise streams are paired across the
five worlds.

Responses use $T=s=6$. The exact condition sets process and sensor noise
to zero. The noisy condition uses 256 repeats per input, intervention mode
and insertion time, sensor standard deviation 0.01, and process standard
deviation 0.002 under the correlated trial-mean acquisition model of
Appendix~\ref{app:v6acquisition}. The total is 98,304 trial equivalents,
half assigned to shared baseline responses. There are
$16\times5\times2=160$ acquisitions: eighty exact and eighty noisy.

The primary fit uses $L_{\rm nom}$ regardless of the true world. The
\emph{true-readout diagnostic} uses the same measured responses, $K$ and
unchanged estimator but supplies $L_c$. It is an oracle calibration check,
not a method for estimating $L_c$. These two readout conditions produce 320
fits. They coincide exactly in the nominal world.

For each system, 32 held-out distributed initial states generate eight-step
trajectories. Initial-state inference uses sensor histories and the fitted
dynamics, never the true initial states. Held-out histories are exact in the
exact condition and receive independent sensor noise of standard deviation
0.001 in the noisy condition. Errors are $e_F,e_z$ from
Eq.~\eqref{eq:v6metrics}; success requires both to be at most 0.10.
Invalid estimates keep infinite errors and remain in all denominators.
Conductivity worlds are paired repetitions of sixteen simulated systems,
not eighty independent systems or participants.

\subsection{Exact responses: labels can be correct while dynamics are wrong}

Table~\ref{tab:v6physical-exact} reports nominal-readout fits. Every exact
target assignment is correct, including all four nonnominal worlds.
Nevertheless, none of the nonnominal cases meets the joint recovery
criterion. The bias therefore remains without sampling noise or target-label
errors. The true-readout diagnostic recovers all eighty exact
system--conductivity cases; its maximum relative dynamics error is
$1.02\times10^{-14}$.

\begin{table}[htbp]
\centering\small
\caption{Exact conditional mean responses with the nominal leadfield supplied
to the estimator. Each row contains the same sixteen paired anatomy--seed
cases. Joint recovery requires $e_F,e_z\le10\%$. Error columns are medians
across the sixteen cases and are expressed in percent. Conductivity is an
assumed physical parameter, not an estimated participant property.}
\label{tab:v6physical-exact}
\begin{tabular}{lrrrr}
\toprule
Skull conductivity (S/m) & Correct labels & Recovery & $e_F$ (\%) & $e_z$ (\%)\\
\midrule
0.0030 & 100\% & 0/16 & 26.63 & 34.66\\
0.0045 & 100\% & 0/16 & 11.17 & 15.26\\
0.0060 (nominal) & 100\% & 16/16 & $<10^{-10}$ & $<10^{-10}$\\
0.0090 & 100\% & 0/16 & 15.99 & 23.38\\
0.0120 & 100\% & 0/16 & 27.52 & 40.05\\
\bottomrule
\end{tabular}
\end{table}

An evaluation-only response-scale diagnostic projects each estimated history
onto its true target's history. Its median absolute scale error across modes
within a system, then systems, is 34.58\%, 11.70\%, approximately zero,
11.81\% and 17.63\% across the five worlds. Since the physical sweep changes
direction as well as scale, these measurements do not identify a purely
multiplicative cause of the dynamics error. They show that correct labels
do not guarantee correctly calibrated histories.

\subsection{Noisy responses and retained failures}

Table~\ref{tab:v6physical-noisy} reports the fixed-budget noisy results.
Using the true readout does not eliminate finite-sample failures: success
is thirteen or fourteen of sixteen cases, depending on the world. Using the
nominal readout gives no joint success outside the nominal world.

\begin{table}[htbp]
\centering\small
\caption{Noisy physical-readout sensitivity at 98,304 trial equivalents.
The same sixteen systems are paired across all rows and both readout
conditions. Errors are all-case medians in percent. ``True'' supplies the
generating leadfield to the unchanged blind estimator; it does not supply
targets, changed coefficients or held-out source states.}
\label{tab:v6physical-noisy}
\begin{tabular}{rrrrrrr}
\toprule
 & \multicolumn{3}{c}{Nominal readout} & \multicolumn{3}{c}{True readout}\\
\cmidrule(lr){2-4}\cmidrule(lr){5-7}
$c$ (S/m) & Recovery & $e_F$ (\%) & $e_z$ (\%) & Recovery & $e_F$ (\%) & $e_z$ (\%)\\
\midrule
0.0030 & 0/16 & 27.13 & 34.70 & 13/16 & 4.47 & 3.15\\
0.0045 & 0/16 & 12.61 & 15.76 & 13/16 & 3.64 & 2.45\\
0.0060 & 13/16 & 3.25 & 2.11 & 13/16 & 3.25 & 2.11\\
0.0090 & 0/16 & 16.61 & 23.27 & 14/16 & 2.92 & 1.79\\
0.0120 & 0/16 & 27.41 & 39.76 & 14/16 & 2.77 & 1.64\\
\bottomrule
\end{tabular}
\end{table}

One noisy nominal-readout estimate is invalid, for participant geometry 03,
seed 7101 at $c=0.003$. A corrupted response direction is assigned to a
repeated target, leaving coverage rank eleven. It remains in all-case
summaries. Noisy nominal-readout target accuracy is 127/128 at this
conductivity and 128/128 at every other conductivity. The systematic
recovery error therefore persists mostly without mislabeling, although noise
can also cause a discrete target error.

\subsection{Numerical audit and interpretation}

Independent numerical verification passed twenty physical-gain contraction
checks, recomputed all 320 fits and their metrics from saved response blocks,
regenerated all 160 acquisitions, and checked 2,560 intervention-level
response-scale statistics. Maximum discrepancies were
$5.69\times10^{-14}$ for gain contraction,
$7.22\times10^{-16}$ for fit metrics and
$6.67\times10^{-16}$ for scale diagnostics. All eighty exact true-readout
cases recovered, and shared source units and nominal/true-readout identity
at $c_0$ were verified.

Geometry-consistency checks use tolerances consistent with the precision
of the saved coordinate transforms. The maximum source-position discrepancy
is 1.85 nm. These checks do not alter the experimental conditions, source
units, response arrays or recovery thresholds.

The study establishes a boundary of the known-$L$ construction. Correct
anatomical labels and repeatable spatial directions need not give accurate
source-coordinate dynamics when the readout calibration is wrong. Supplying
the true simulated-world readout removes that mismatch in exact responses,
but is not a deployable calibration procedure. The study uses four reused
anatomies, twelve preselected sources, controlled source initialization and
one assumed conductivity family. It does not demonstrate whole-brain
localization, estimate actual human conductivity errors or recover neural
$F$ from recorded EEG. Proposition~\ref{prop:v6stability} assumes the correct
readout and must not be read as a robustness guarantee for this mismatch.

\section{Experimental interpretation and additional related work}
\label{app:v6applicability}

This appendix explains the experimental meaning of the information assumed by the recovery result and its relation to existing work. The inference target is source-coordinate dynamics under a specified perturbational model. Neither the presence of a stimulation event nor accurate prediction of its measured response establishes all the assumptions of that model.

\subsection{What the result means for source localization}

Our inference target is $F$ among predefined anatomical sources, not their unrestricted positions or a whole-brain map of arbitrary generators. The known forward model fixes locations, orientations and source units. Local response contrasts assign dynamical histories to those coordinates, after which informative temporal observations can support source-trajectory reconstruction. This distinguishes our question from both single-snapshot localization and learning a latent state-space model up to an arbitrary coordinate transformation.

Accurate sensor prediction, repeatable spatial patterns and correct source-coordinate dynamics are different outcomes. The simulations evaluate recovery against known dynamics and source trajectories. The recorded-current EEG analyses instead test spatial-response transfer and anatomical agreement in measured signals; they establish neither the proposed neural mechanism nor recovery of true neural $F$. These distinct scopes are retained when interpreting the experimental evidence in the main text.

\subsection{Experimental meaning of the assumptions}

The method concerns calibrated perturbation experiments rather than arbitrary passive recordings. Its practical interpretation is a set of information and measurement requirements, each with a distinct role.

The leadfield $L$ must describe the measured channels in the same source coordinates and units used by the dynamics. Knowing anatomical labels alone is insufficient, as the conductivity study shows. Similarly, $K$ specifies mean source-state patterns, not the physical settings of a stimulator. A known stimulation location, current or coil position does not establish the resulting $Ku$. The present recovery guarantee assumes both maps; it does not estimate an unknown neural initialization map from recorded EEG.

The term $D_ez_\tau$ depends on the current state and changes one source's receiving rule. The distinction between driving inputs and modulation of coupling is established in dynamic causal modelling \citep{friston2003}. Our theorem adds much narrower restrictions: only one receiving row changes, only one transition differs, and the same baseline dynamics resume afterwards. A focal additive pulse can also produce a rank-one subsequent response, so rank one alone cannot establish that a mechanism changed. Moreover, a local change in a continuous-time update rule does not generally remain a single-row change after evolution over a finite sampling interval. A biological description as ``local stimulation'' therefore does not by itself establish Eq.~\eqref{eq:v6row}.

In our simulations, one transition is a step of the discrete dynamical model; we do not assign it a physiologically calibrated duration in milliseconds. A physical implementation would need to specify this interval separately from the recording sampling period and establish that the perturbation affects only one model transition, after which baseline dynamics resume. Increasing the recording rate alone does not establish this property, because activity can propagate through the network while the perturbation is still acting.

Matched response contrasts require the same calibrated initializations, insertion times and observation windows under a shared baseline $F$. Arbitrary differences between sessions do not automatically provide $H^{[e]}-H_+$. The first post-transition sensor block also has a special role: it anchors the recovered direction to $Le_j$. If observations begin $d\ge1$ baseline steps later, that first available block is $LF^de_j$, not $Le_j$; the stated anatomical calibration cannot simply be applied unchanged to the delayed history.

\subsection{Relation to other perturbational recordings}

A controlled neuronal preparation offers a candidate setting for examining the intervention assumptions. Dynamic clamp computes injected current from recorded neuronal activity and can introduce artificial conductances or synaptic connections \citep{sharp1993dynamic}. Optogenetic experiments have also demonstrated modulation of neuronal response gain \citep{wilson2012division}. A candidate protocol would compare matched initializations with and without a brief state-dependent perturbation at one recorded site, followed by an unperturbed observation period. Direct recordings could help assess the affected state directions and calibrate initialization responses. Neither technique alone guarantees a single receiving-row change in the sampled dynamics. Moreover, simply suppressing a source that receives no baseline excitation need not reveal a missing direction: the perturbation must produce a nonzero response contrast.

Human intracranial stimulation provides a complementary application context. Simultaneous stereo-EEG (SEEG) and scalp EEG recordings measure local and distributed responses to electrical pulses \citep{parmigiani2022simultaneous}. Such recordings motivate anatomical interpretation of perturbation responses, but a known stimulation contact does not establish a calibrated neural initialization or a change in one receiving mechanism. A focal additive pulse can also generate a source-specific response history; this does not by itself validate the mechanism-change model studied here.

Transcranial magnetic stimulation (TMS) is a perturbation modality that can be combined with EEG, rather than an alternative readout. Paired-pulse TMS--EEG demonstrates conditioning-dependent responses and the importance of correcting for the conditioning pulse's own response \citep{premoli2014paired}. It provides physiological motivation for controlled response comparisons, but does not establish our source-level $K$ or single-transition receiving-row model. Early TMS--EEG measurements can contain substantial muscle artifacts \citep{mutanen2013artifacts}; discarding those samples may remove the very block used for our anatomical calibration. The current results consequently do not establish a ready-to-apply estimator for these recordings.

The algebraic reconstruction is not specific to electrical potentials: it uses a known linear observation map. Magnetoencephalography (MEG) is a closely related observation setting, while source-resolved optical perturbation studies illustrate the value of informative excitation \citep{zhang2010,wagenmaker2024active}. Nevertheless, changing modality does not remove the initialization, timing and coverage requirements. In particular, calcium indicators introduce temporal response dynamics, and group photostimulation is not generally a single-source receiving-mechanism change. Applicability to another physiological signal is conditional on its observation and intervention models satisfying the stated assumptions; it is not established merely by the presence of hidden sources and observable measurements.

\subsection{Additional related work}
\label{app:v6related}

Observational time series can identify linear latent dynamics under appropriate assumptions on mixing and non-Gaussian process noise \citep{zhang2011}. Unknown-target interventions identify linear causal representations under assumptions on mixing and intervention laws \citep{squires2023,acarturk2024}; broader results consider nonlinear mixing and general environments \citep{zhang2024,ng2025}. Higher-order cumulants address more latent variables than observations \citep{leyes2025}, and temporal causal representation learning also treats noninvertible observation maps \citep{chen2024}. Our construction uses known anatomy and controlled mean responses to identify dynamics in predefined source labels and units, with unknown local changes supplying otherwise missing source directions.

Classical realization constructs dynamics from input--output responses \citep{ho1966}; structured and switched identification study physical constraints and shared state coordinates \citep{yu2015,petreczky2010}. \citet{rajendran2024} vary control distributions in Gaussian linear systems; their full-system recovery result requires baseline responses to span the state space and retains coordinate ambiguities. Our construction extracts anatomically calibrated response columns directly from receiving-mechanism contrasts, without full baseline reachability. Finite-sample realization analyses \citep{oymak2019,sun2020} address statistical estimation; our stability analysis instead bounds deterministic perturbations conditional on correct anatomical labels.

Causal EEG and MEG analysis connects source separation to directed mechanisms \citep{zhang2010}. However, source reconstruction alone does not establish valid interaction estimates: simulations have demonstrated spurious connectivity under volume conduction and noise for conventional sensor-space and source-space analyses \citep{haufe2013critical}. Dynamic causal modelling distinguishes direct driving inputs from input-dependent modulation of neuronal coupling \citep{friston2003}. Our receiving-row change belongs conceptually to the modulation class, but imposes a more restrictive locality and timing model. We establish when such a change reveals an anatomically identifiable response history and when those histories determine the baseline dynamics. Neither the idea of mechanism modulation nor source localization by itself is the contribution.

\citet{wagenmaker2024active} use two-photon holographic optogenetics and calcium imaging to study informative stimulation patterns for neural population dynamics. They evaluate known dynamics in a data-fitted simulator and predictive performance on real recordings, where the true dynamics are unavailable. Their stimulus enters as an additive input, rather than our unknown receiving-row change. Together with the conditioning-dependent response comparisons in paired-pulse TMS--EEG \citep{premoli2014paired}, this work motivates the importance of informative perturbations and matched comparisons. These studies do not establish the calibrated initialization map or single-row, single-transition assumptions used here. Accordingly, our simulations evaluate source-coordinate recovery, whereas our recorded-current analyses evaluate spatial-response transfer and anatomical calibration.

The practical lesson is to distinguish missing excitation, weak observation and calibration error. Exposed source directions address the first, informative temporal measurements address the second, and repetition reduces response noise without resolving systematic model mismatch. These distinctions remain relevant even for an invertible leadfield. Our evidence establishes recovery and its boundaries within the specified perturbational model, with a separate measurement-level assessment on recorded-current EEG. It does not establish neural connectivity recovery from arbitrary biological recordings, nor an unconditional statistical guarantee covering target selection and numerical-rank decisions.

\end{document}